\documentclass[review,onefignum,onetabnum]{siamonline190516}

\usepackage{lipsum}
\usepackage{amsfonts}
\usepackage{graphicx}
\usepackage{epstopdf}
\usepackage{algorithmic}
\ifpdf
  \DeclareGraphicsExtensions{.eps,.pdf,.png,.jpg}
\else
  \DeclareGraphicsExtensions{.eps}
\fi
\input{mysymbol.sty}

\usepackage{enumitem}
\setlist[enumerate]{leftmargin=.5in}
\setlist[itemize]{leftmargin=.5in}

\newsiamremark{remark}{Remark}
\newsiamremark{example}{Example}
\newsiamremark{hypothesis}{Hypothesis}
\crefname{hypothesis}{Hypothesis}{Hypotheses}
\newsiamthm{claim}{Claim}

\headers{Robust Hierarchical Clustering for Directed Networks}{G. Carlsson, F. M\'emoli, and S. Segarra}

\title{Robust Hierarchical Clustering for Directed Networks: An Axiomatic Approach\thanks{Authors are sorted in alphabetical order.
\funding{This work was funded by NSF RI 1422400, NSF DMS 1723003, and NSF CCF 1740761 (F. M\'emoli).}}}

\author{Gunnar Carlsson\thanks{Dept.\! of Mathematics, Stanford University 
  (\email{gunnar@math.stanford.edu}).}
\and Facundo M\'emoli\thanks{Depts.\! of Mathematics and of Computer Science and Eng., Ohio State University
  (\email{memoli@math.osu.edu}).}
\and Santiago Segarra\thanks{Dept.\! of Electrical and Computer Eng., Rice University 
	(\email{segarra@rice.edu}).} }

\usepackage{amsopn}

\usepackage{needspace}

\newcommand{\myparagraph}[1]{\needspace{1\baselineskip}\medskip\noindent {\it #1.}}

\newenvironment{indentedparagraph}[1] 
{\begin{list}{}%
         {\setlength{\leftmargin}{11pt}
          \setlength{\topsep}{10pt}}
         \item[]{\it #1.}}
{\end{list}}

\usepackage{tikz}
\usepackage{framed}
\usepackage[normalem]{ulem}

\usetikzlibrary{shapes,arrows}
\tikzstyle{phantom vertex} = [ ellipse, 
                               anchor = center, 
                               minimum height = 1*\unit, 
                               minimum width  = 1*\unit,
                               inner sep=0pt,
                               anchor=center]
\tikzstyle{red vertex}   = [black, fill = red!20,   phantom vertex, draw]
\tikzstyle{black vertex} = [black, fill = black!20, phantom vertex, draw]
\tikzstyle{blue vertex}  = [black, fill = blue!20,  phantom vertex, draw]
\tikzstyle{green vertex} = [black, fill = green!20,  phantom vertex, draw]
\tikzstyle{vertex}       = [draw, phantom vertex]

\tikzstyle{point} = [ellipse, inner sep=0pt, draw, fill=white, anchor = center,
                     minimum height = 0.05*\unit, minimum width  = 0.05*\unit]

\def \sep {\text{\normalfont sep}}
\def \dom {\text{\normalfont dom}}
\def\R {\text{\normalfont R}}

\def\NR {\text{\normalfont NR}}
\def\SR {\text{\normalfont SR}}
\def\SL {\text{\normalfont SL}}

\newcommand{\QED}{\hfill\ensuremath{\blacksquare}}

\begin{document}

\maketitle

\begin{abstract}
  We provide a complete taxonomic characterization of robust hierarchical clustering methods for directed networks following an axiomatic approach. 
  We begin by introducing three practical properties associated with the notion of robustness in hierarchical clustering: linear scale preservation, stability, and excisiveness. 
  Linear scale preservation enforces imperviousness to change in units of measure whereas stability ensures that a bounded perturbation in the input network entails a bounded perturbation in the clustering output.
  Excisiveness refers to the local consistency of the clustering outcome. 
  Algorithmically, excisiveness implies that we can reduce computational complexity by only clustering a subset of our data while theoretically guaranteeing that the same hierarchical outcome would be observed when clustering the whole dataset.
  In parallel to these three properties, we introduce the concept of representability, a generative model for describing clustering methods through the specification of their action on a collection of networks.
  Our main result is to leverage this generative model to give a precise characterization of all robust -- i.e., excisive, linear scale preserving, and stable  -- hierarchical clustering methods for directed networks.
  We also address the implementation of our methods and describe an application to real data.
\end{abstract}

\section{Introduction} \label{sec_introduction}

The concept of clustering, i.e., partitioning a dataset into groups such that objects in one group are more similar to each other than they are to objects outside the group, is a fundamental tool for the advancement of knowledge in a wide range of disciplines from, e.g., medicine \cite{WalshRybicki06} to marketing \cite{PunjStewart83}. 
Motivated by its relevance, literally hundreds of clustering algorithms have been developed in the past decades \cite{clusteringref, lance67general, GirvanNewman02, GirvanNewman04,  NgEtal02, ShiMalik00, spectral-clustering} mainly for the application to finite metric spaces but also for the increasingly relevant case of directed networks \cite{SaitoYadohisa04}, in which the dissimilarity from node $x$ to node $x'$ may differ from the one from $x'$ to $x$ \cite{boyd-asymmetric, hubert-min,MeilaPentney07,murtagh-multidimensional,PentneyMeila05,slater1976hierarchical,slater1984partial,tarjan-improved,ZhaoKarypis05}. 
Directionality naturally arises in multiple contexts~\cite{mackay_2020_directed,marques_2020_signal}.
Apart from the canonical example of a food web, information networks such as scientific citations or the world wide web are typically directed.
Gene-regulatory networks are highly nonreciprocal and this lack of reciprocity needs to be accounted for when, e.g., grouping (clustering) genes that might have similar functional properties~\cite{yu_2006_genomic}.
Moreover, in social networks, pairwise relations are rarely purely symmetric and this asymmetry is key to accurately separate leaders from followers~\cite{segarra_2017_blind}.
In these settings, effectively extracting knowledge from real and noisy data requires a theoretical understanding of robust techniques to analyze directed networks.

Although the theoretical underpinnings of clustering are not in general as well developed as its practical applications \cite{sober,science_art,vonlux-david}, the foundations of clustering in metric spaces have been developed over the past two decades~\cite{ben-david-ackermann, clust-um, CarlssonMemoli10, kleinberg, VanLaarhoven14, Meila05,ben-david-reza}. 
For the specific case of hierarchical clustering \cite{clusteringref, lance67general,ZhaoKarypis05} where, instead of a single partition we look for a family of partitions indexed by a resolution parameter, some theoretical understanding has been achieved for the case of finite metric spaces \cite{clust-um,CarlssonMemoli10} and for the more general case of directed networks \cite{Carlssonetal13_3, Carlssonetal13,Carlssonetal14, Carlsson_2018_hierarchical,memoli2020motivic}. 
Of special interest to our work is \cite{Carlsson_2018_hierarchical}, where two axioms encoding desirable features of hierarchical clustering methods were proposed, and an infinite but bounded family of methods satisfying these axioms (denominated the family of admissible methods) was identified. 
However, the disadvantage of this approach is that these two axioms are not sufficient to ensure robustness of the clustering methods abiding by them.

In the current paper, we build upon \cite{CarlssonMemoli10}, \cite{Carlssonetal13}, and~\cite{Carlsson_2018_hierarchical}, and deepen the characterization of hierarchical clustering methods on directed networks to identify those with robustness properties, which we view as encoding practical relevance. 
We define and analyze three properties of practical relevance, namely \emph{excisiveness, linear scale invariance, and stability}, and say that a \emph{hierarchical clustering method is robust if it possesses these three properties}.

{\bf Contributions.}
The contributions of this paper are threefold:\\
	i) We introduce a formal definition of robustness for hierarchical clustering methods based on the properties of excisiveness, linear scale preservation, and stability. Furthermore, we determine a subset of robust methods among a set of established admissible methods.\\
	ii) We introduce the notion of representability as a generative model for hierarchical clustering methods, where a method is defined through the specification of its local behavior. We also show that every representable method can be factorized into two well-defined operations: a symmetrizing operation followed by a well-established hierarchical clustering method.\\
	iii)  We relate the two aforementioned notions by showing that, within the so-called admissible methods, representability is equivalent to robustness. This novel characterization result implies that any admissible and robust clustering method can be generated using a collection of representers.

{\bf Related work.}
The study of robust clustering methods has been an active area of research for a few decades now~\cite{ana_2003_robust, dave_1997_robust, garcia_2010_review}.
If we focus on the specific case of hierarchical clustering, it has long been known that certain linkage functions -- such as Ward's linkage~\cite{ward_1963_hierarchical} -- are more tolerant to noise in the input data than others. 
Moreover, relatively ad hoc methods -- like Wishart's method~\cite{wishart-1969} -- have been developed to avoid shortcomings of other linkage functions by discarding low-density regions of the data.
More recently, improvements over these classical methods~\cite{chaudhuri_2010_rates}, novel schemes that specifically focus on the imperviousness to outliers~\cite{guha_1998_cure, lattanzi_2015_robust}, and robust algorithms that can accommodate categorical attributes~\cite{guha_2000_rock} have been developed.
Active methods have also been proposed where the similarities are selectively sampled before linkage, thus gaining robustness to a limited fraction of anomalous similarities~\cite{eriksson_2011_active}.
From a more theoretical viewpoint, \cite{balcan_2014_robust} proposes a robust linkage function to provably cluster data that satisfies a `good neighborhood' property, which is a relaxation of the strict separation property (where all points are more similar to points in their own cluster than to points in any other cluster).
The current paper has two main differences with all the aforementioned body of work: 
\begin{itemize}
    \item[-] The definition of robustness here considered is novel and precisely defined based on the three properties of excisiveness, linear scale preservation, and stability.
    \item[-] The result of our theoretical study is not just a single hierarchical clustering method that can be shown to be robust but rather \emph{a generative model to construct all possible robust hierarchical clustering methods under the considered framework.}
\end{itemize}

\vspace{1mm}
{\bf Paper outline.}
After introducing basic concepts about clustering and networks (Section~\ref{sec_preliminaries}), in Section~\ref{S:robust} we present a formal definition of robustness by introducing the properties of excisiveness (Section~\ref{Ss:excisiveness}), linear scale preservation (Section~\ref{Ss:linear_scale_preservation}), and stability (Section~\ref{Ss:stability}).
Representability, a notion introduced in Section \ref{sec_representability}, provides a generative model for clustering where a method is defined through the specification of its behavior in a collection of special networks. 
In Section~\ref{sec_representability_and_single_linkage}, we show that every representable clustering method can be decomposed into a symmetrizing operation followed by the application of single linkage clustering. 
Although seemingly unrelated with the robustness properties previously mentioned, representability is a key notion to characterize clustering methods. 
Indeed, in Section \ref{subsec_generative_model} we present our main result stating that an admissible clustering method is representable if and only if it is robust. 
Finally, in Section \ref{sec_numerical_experiments}, we illustrate the main result by implementing a representable clustering method, testing it on a real-world economic network, and confirming its robustness.

\section {Preliminaries}\label{sec_preliminaries}

A network $N$ is defined as a pair $(X,A_X)$ where $X$ is a finite set of $n$ points or nodes and $A_X: X \times X \to \reals_+$ is a dissimilarity function. Dissimilarities $A_X(x,x')$ from $x$ to $x'$ are non-negative, and $0$ if and only if $x=x'$, but may not satisfy the triangle inequality and may be directed or asymmetric, i.e., $A_X(x,x') \neq A_X(x',x)$ for some $x, x' \in X$. Given a positive real $\alpha$, define the multiple of a network $\alpha * N := (X, \alpha \, A_X)$. Let $\ccalN$ denote the collection of all networks. Networks $N\in\ccalN$ can have different node sets $X$ and different dissimilarities $A_X$. We focus our study on directed networks since these general structures include, as particular cases, undirected networks and finite metric spaces.

The output of hierarchically clustering the network $N=(X,A_X)$ is a dendrogram $D_X$, that is a collection of partitions $D_X(\delta)$ indexed by the resolution parameter $\delta\geq 0$ satisfying the following conditions:
\begin{itemize}
    \item[-] The partitions in $D_X(\delta)$ are nested, i.e., if $x$ and $x'$ are in the same partition at resolution $\delta$, then they stay co-clustered for all larger resolutions $\delta' > \delta$.
    \item[-] $D_X(0) = \big\{ \{x\}, \, x\in X\big\}$, i.e., for the resolution parameter $\delta=0$ each point $x \in X$ must form its own cluster.
    \item[-] $D_X(\delta_0) = \big\{ X \big\}$, i.e., for some sufficiently large resolution $\delta_0$ all nodes must belong to the same cluster.
\end{itemize}
From these fundamental requirements and a technical condition of continuity (for all $\delta$ there exists $\epsilon>0$ such that $D_X(\delta)=D_X(\gamma)$ for $\gamma \in [\delta, \delta+\epsilon]$) it follows that dendrograms can be represented as trees \cite[Sec. 3.1]{clust-um}. 
The interpretation of a dendrogram is that of a structure which yields different clusterings at different resolutions. 
When $x$ and $x'$ are co-clustered at resolution $\delta$ in $D_X$ we say that they are equivalent at that resolution and write $x\sim_{D_X(\delta)} x'$. 

Given a network $(X, A_X)$ and $x, x' \in X$, a chain $C(x, x')$ is an ordered sequence of nodes in $X$, $C(x, x')=[x=x_0, x_1, \ldots , x_{l-1}, x_l=x']$,
which starts at $x$ and finishes at $x'$.
The \emph{links} of a chain are the edges connecting consecutive nodes of the chain in the direction given by it. We define the \emph{cost} of chain $C(x, x')$ as the maximum dissimilarity $\max_{i | x_i\in C(x,x')}A_X(x_i,x_{i+1})$ encountered when traversing its links in order. The directed minimum chain cost $\tdu^*_X(x, x')$ between $x$ and $x'$ is defined as the minimum cost among all the chains connecting $x$ to $x'$,
\begin{align}\label{eqn_nonreciprocal_chains} 
\tdu^*_X(x, x') := \min_{C(x,x')} \,\,
\max_{i | x_i\in C(x,x')} A_X(x_i,x_{i+1}).
\end{align} 
An ultrametric $u_X$ on the set $X$ is a function $u_X:X \times X \to \reals_+$ that satisfies symmetry $u_X(x, x')=u_X(x', x)$, identity $u_X(x, x')=0 \iff x=x'$ and the strong triangle inequality 
\begin{equation}\label{eqn_strong_triangle_inequality}
u_X(x,x') \leq \max \big(u_X(x,x''),  u_X(x'',x') \big),
\end{equation}
for all $x, x', x'' \in X$.
For a given dendrogram $D_X$, consider the minimum resolution at which $x$ and $x'$ are co-clustered and define 
\begin{equation}\label{eqn_theo_dendrograms_as_ultrametrics_10}
u_X(x,x') := \min \big\{ \delta\geq 0 \, |\,  x\sim_{D_X(\delta)} x' \big\}.
\end{equation}
It can be shown that the function $u_X$ as defined in \eqref{eqn_theo_dendrograms_as_ultrametrics_10} is an ultrametric on the set $X$, from where it follows that dendrograms and finite ultrametric spaces are equivalent~\cite{clust-um}. However, ultrametrics are more convenient than dendrograms for the results developed in this paper.

A hierarchical clustering method is defined as a map $\ccalH:\ccalN \to \ccalD$ from the collection of networks $\ccalN$ to the collection of all dendrograms $\ccalD$, or, equivalently, as a map $\ccalH:\ccalN \to \ccalU$ mapping every (possibly directed) network into the collection $\ccalU$ of networks with ultrametrics as dissimilarity functions. 

This loose definition of a hierarchical clustering method allows the existence of a wide variety of methods, most of them of little practical utility. Thus, in Section~\ref{subsec_admissible_hierarchical_clustering_algorithms} we recall an axiomatic construction built to select a subfamily of admissible clustering methods. 

{For future reference, we say that two methods $\ccalH$ and $\ccalH'$ are \emph{equivalent}, denoted $\ccalH\equiv\ccalH'$, if $\ccalH(N)=\ccalH'(N)$ for all networks $N\in \ccalN.$} 
{We also recall the definition of single linkage hierarchical clustering $\ccalH^\SL$ of \emph{symmetric or undirected} networks with output ultrametrics $u_X^\SL(x,x'):=\min_{C(x,x')}\max_{i}A_X(x_i,x_{i+1})$.}

\subsection{Admissible hierarchical clustering methods}\label{subsec_admissible_hierarchical_clustering_algorithms}

%
\begin{figure}
	\centering
	\centerline{\def \thisplotscale {0.5}
\def \unit {\thisplotscale cm}

{\small
\begin{tikzpicture}[-stealth, shorten >=2 ,scale = \thisplotscale, font=\footnotesize]

	\node[blue vertex] (x) {$x$} ;

	\path (x) ++ (4,0)   node [blue vertex]    (x1)   {$x_1$} 
	          ++ (4,0)   node [phantom vertex] (x2)   {$\ldots$}
	          ++ (0.5,0) node [phantom vertex] (xlm1) {$\ldots$} 
	          ++ (4,0)   node [blue vertex]    (xl)   {$x_{l-1}$}
	          ++ (4,0)    node [blue vertex]   (xp)   {$x'$}; 

	\path (x)    edge [bend left, above] node {$A_X(x,x_1)$}       (x1);	
	\path (x1)   edge [bend left, above] node {$A_X(x_1,x_2)$}     (x2);	
	\path (xlm1) edge [bend left, above] node {$A_X(x_{l-2},x_{l-1})$} (xl);	
	\path (xl)   edge [bend left, above] node {$A_X(x_{l-1},x')$}      (xp);	

	\path (x1)   edge [bend left, below] node {$A_X(x_1,x)$}       (x);
	\path (x2)   edge [bend left, below] node {$A_X(x_2,x_1)$}     (x1);
	\path (xl)   edge [bend left, below] node {$A_X(x_{l-1},x_{l-2})$} (xlm1);
	\path (xp)   edge [bend left, below] node {$A_X(x',x_{l-1})$}      (xl);

\end{tikzpicture}
} }
	\vspace{-0.1in}
	\caption{Reciprocal clustering. Nodes $x,x'$ cluster at resolution $\delta$ if they can be joined with a bidirectional chain of maximum dissimilarity $\delta$ [cf. \eqref{eqn_reciprocal_clustering}].}
	\vspace{-0.05in}
	\label{fig_reciprocal_path}
\end{figure}
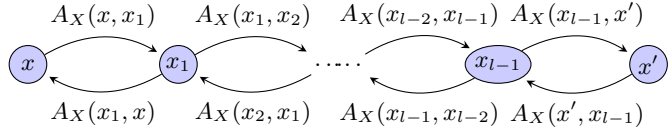

In \cite{Carlssonetal13, Carlsson_2018_hierarchical}, the authors impose the following two requirements on clustering methods:

\myparagraph{(A1) Axiom of Value} Given a two-node network $N=(\{p,q\},A_{p,q})$ with $A_{p,q}(p,q)=\alpha$, and  $A_{p,q}(q,p)=\beta$, the ultrametric $(X,u_{p,q})=\ccalH(N)$ output by $\ccalH$ satisfies
\begin{equation}\label{eqn_two_node_network_ultrametric}
u_{p,q}(p,q) = \max(\alpha,\beta).
\end{equation}
\myparagraph{(A2) Axiom of Transformation} Given networks $N_X=(X,A_X)$ and $N_Y=(Y,A_Y)$ and a dissimilarity reducing map $\phi:X\to Y$, i.e. a map $\phi$ such that for all $x,x' \in X$ it holds $A_X(x,x')\geq A_Y(\phi(x),\phi(x'))$, the outputs $(X,u_X)=\ccalH(N_X)$ and $(Y,u_Y)=\ccalH(N_Y)$ satisfy 
\begin{equation}\label{eqn_dissimilarity_reducing_ultrametric}
u_X(x,x') \geq u_Y(\phi(x),\phi(x')).
\end{equation} 

\medskip\noindent We say that node $x$ is able to influence node $x'$ at resolution $\delta$ if the dissimilarity from $x$ to $x'$ is not greater than $\delta$. In two-node networks, our intuition dictates that a cluster is formed if nodes $p$ and $q$ are able to influence each other. Thus, axiom (A1) states that in a network with two nodes, the dendrogram $D_X$ has them merging at the maximum value of the two dissimilarities between them. 
Axiom (A2) captures the intuition that if a network is transformed such that some nodes become more similar but no pair of nodes increases its dissimilarity, then the transformed network should cluster at lower resolutions than the original one. Formally, (A2) states that a contraction of the dissimilarity function $A_X$ entails a contraction of the associated ultrametric $u_X$.

A hierarchical clustering method $\ccalH$ is \emph{admissible} if it satisfies axioms (A1) and (A2). Two admissible methods of interest are reciprocal and nonreciprocal clustering, as defined next. 
\begin{indentedparagraph}{Reciprocal and nonreciprocal clustering} 
The \emph{reciprocal} clustering method $\ccalH^{\R}$ outputs the ultrametric $(X,u^{\R}_X)=\ccalH^{\R}(X,A_X)$ defined as 
\begin{align}\label{eqn_reciprocal_clustering} 
u^{\R}_X(x,x')
&:= \min_{C(x,x')} \, \max_{i | x_i\in C(x,x')}
\bbarA_X(x_i,x_{i+1}),
\end{align}
where $\bbarA_X(x,x'):=\max(A_X(x,x'), A_X(x',x))$ for all $x, x' \in X$.
The \emph{nonreciprocal} clustering method $\ccalH^{\NR}$ outputs the ultrametric $(X,u^{\NR}_X)=\ccalH^{\NR}(X,A_X)$ given by 
\begin{align}\label{eqn_nonreciprocal_clustering} 
u^{\NR}_X(x,x') := \max \Big( \tdu^{*}_X(x,x'),\ \tdu^{*}_X(x',x )\Big).
\end{align} 
\end{indentedparagraph}

\noindent 
Intuitively, in \eqref{eqn_reciprocal_clustering} we search for chains $C(x,x')$ linking nodes $x$ and $x'$. 
Then, for a given chain, we walk from $x$ to $x'$ and determine the maximum dissimilarity, in either the forward or backward direction, across all links in the chain. 
The reciprocal ultrametric $u^{\R}_X(x,x')$ is the minimum of this value across all possible chains; see Fig.~\ref{fig_reciprocal_path}. 
Putting it differently, reciprocal clustering joins $x$ and $x'$ at resolution $\delta$ if it is possible to go back and forth at maximum cost $\delta$ through the same chain. 
By contrast, nonreciprocal clustering permits different chains. 
In \eqref{eqn_nonreciprocal_clustering}, we implicitly consider forward chains $C(x,x')$ going from $x$ to $x'$ and backward chains $C(x',x)$ from $x'$ to $x$. 
We then determine the respective maximum dissimilarities and search independently for the forward and backward chains that minimize the respective maximum dissimilarities. 
The nonreciprocal ultrametric $u^{\NR}_X(x,x')$ is the maximum of these two minimum values; see Fig.~\ref{fig_nonreciprocal_path}.

%
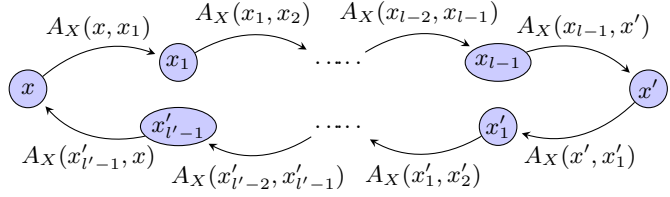
\begin{figure}
	\centering
	\def \thisplotscale {0.5}
\def \unit {\thisplotscale cm}
{\small
\begin{tikzpicture}[-stealth, shorten >=2 ,scale = \thisplotscale, font=\footnotesize]
	
	\node[blue vertex] (x) {$x$} ;

	\path (x) ++ (4,0.7)   node [blue vertex]    (x1)   {$x_1$} 
	          ++ (4,0)   node [phantom vertex] (x2)   {$\ldots$}
	          ++ (0.5,0) node [phantom vertex] (xlm1) {$\ldots$} 
	          ++ (4,0)   node [blue vertex]    (xl)   {$x_{l-1}$}
	          ++ (4,-0.7)  node [blue vertex]    (xp)   {$x'$}; 

	\path (x) ++ (4,-1)  node [blue vertex]    (xlp)   {$x'_{l'-1}$} 
	          ++ (4,0)   node [phantom vertex] (xlm1p) {$\ldots$}
	          ++ (0.5,0) node [phantom vertex] (x2p)   {$\ldots$} 
	          ++ (4,0)   node [blue vertex]    (x1p)   {$x'_{1}$}; 

	\path (x)    edge [bend left, above] node {$A_X(x,x_1)$}       (x1);	
	\path (x1)   edge [bend left, above] node {$A_X(x_1,x_2)$}     (x2);	
	\path (xlm1) edge [bend left, above] node {$A_X(x_{l-2},x_{l-1})$} (xl);	
	\path (xl)   edge [bend left, above] node {$A_X(x_{l-1},x')$}      (xp);	

	\path (xp)    edge [bend left, below] node {$A_X(x',x'_1)$}       (x1p);
	\path (x1p)   edge [bend left, below] node {$A_X(x'_1,x'_2)$}     (x2p);
	\path (xlm1p) edge [bend left, below] node {$A_X(x'_{l'-2},x'_{l'-1})$} (xlp);
	\path (xlp)   edge [bend left, below] node {$A_X(x'_{l'-1},x)$}       (x);

\end{tikzpicture}
}
	\vspace{-0.15in}
	\caption{Nonreciprocal clustering. Nodes $x, x'$ cluster at resolution $\delta$ if they can be joined in both directions with possibly different chains of maximum dissimilarity $\delta$ [cf. \eqref{eqn_nonreciprocal_clustering}].}
	\vspace{-0.05in}
	\label{fig_nonreciprocal_path}
\end{figure}

These two methods exemplify extremal behaviors. Indeed, reciprocal and nonreciprocal clustering bound the ultrametrics generated by all admissible methods, as stated next.

%
\begin{theorem}  \label{theo_extremal_ultrametrics} 
{\normalfont (\cite[Theorem 4]{Carlsson_2018_hierarchical})} Consider an arbitrary network $N=(X,A_X)$ and let $u^{\R}_X$ and $u^{\NR}_X$ be the associated reciprocal and nonreciprocal ultrametrics as defined in \eqref{eqn_reciprocal_clustering} and \eqref{eqn_nonreciprocal_clustering}. Then, for any admissible method $\ccalH$ the output ultrametric $(X,u_X)=\ccalH(X,A_X)$ is such that for all pairs $x,x'$,
	\begin{equation}\label{eqn_theo_extremal_ultrametrics} 
	u^{\NR}_X(x,x') \leq  u_X(x,x') \leq  u^{\R}_X(x,x').
	\end{equation} 
	In particular, $u^{\NR}_X=u^{\R}_X$ whenever $(X, A_X)$ is undirected.
\end{theorem}

\smallskip\noindent According to Theorem 1, nonreciprocal clustering yields uniformly minimal ultrametrics while reciprocal clustering yields uniformly maximal ultrametrics among all methods satisfying (A1)-(A2). Moreover, the existence of admissible methods strictly different from $\ccalH^{\NR}$ and $\ccalH^{\R}$ has been shown \cite[Prop. 1]{Carlssonetal13_2}. For symmetric networks, reciprocal and nonreciprocal clustering coincide, implying that there is a unique admissible method, which is equivalent to the well-known single linkage hierarchical clustering method \cite[Ch. 4]{clusteringref}. In Section~\ref{S:robust}, we present practical properties -- excisiveness, linear scale preservation, and stability -- which are not shared by every admissible method and we use them to further winnow the set of clustering methods.

\section{Robust hierarchical clustering methods}\label{S:robust}
We formalize the notion of a robust hierarchical clustering method as one that satisfies three properties: excisiveness (Section~\ref{Ss:excisiveness}), linear scale preservation (Section~\ref{Ss:linear_scale_preservation}), and stability (Section~\ref{Ss:stability}). 
In this section, we define, analyze, and give examples of methods satisfying these three properties.

\subsection{Excisiveness}\label{Ss:excisiveness}

Consider a clustering method $\ccalH$ and a given network $N=(X, A_X)$. Denote by $(X, u_X)=\ccalH(N)$ the ultrametric output, as $D_X$ the output dendrogram and, for a given resolution $\delta$, denote the dendrogram's partition by $D_X(\delta)=\{B_1(\delta),\ldots, B_{J(\delta)}(\delta)\}$, where each block $B_i(\delta)$ represents a cluster at resolution $\delta$. Consider then the induced subnetworks $N^\delta_i$ associated with each block $B_i(\delta)$ of $D_X(\delta)$ defined as 
\begin{align}\label{eqn_excisiveness_subnetworks}
N^\delta_i := \left(B_i(\delta),\ A_X \big|_{B_i(\delta) \times B_i(\delta)}\right),
\end{align}
where $A_X |_{B_i(\delta) \times B_i(\delta)}$ denotes the restriction of $A_X$ to the nodes in $B_i(\delta)$. 
In terms of ultrametrics, networks $N^\delta_i$ are such that their node set $B_i(\delta)$ satisfies
\begin{alignat}{3}\label{eqn_excisiveness_subnetworks_ultrametrics}
&u_X(x, x') \ \leq\ && \delta,\quad 
&& \text{for all\ } x,x' \in B_i(\delta), \nonumber\\
&u_X(x, x'')\ >   \ && \delta,\quad 
&& \text{for all\ } x \in B_i(\delta),\ 
x'' \notin B_i(\delta).
\end{alignat}
%

%
\begin{figure}
	\centering\def \thisplotscale {0.53}
\def \unit {\thisplotscale cm}

{\small
\begin{tikzpicture}[-stealth, shorten >=2, scale = \thisplotscale]

   \draw [-stealth] (6.5,-2.5) -- (6.5,4.5);
    \draw [-stealth] (6.3,-2.3) -- (12.5,-2.3);
    \draw [-, draw=black!30] (9.5,-2.3) -- (9.5,4.3);
    
 \draw[thick, -] (6.5, -2) -- ++(1,0) -- ++(0 , 0.5 ) -- ++(-1, 0) -- ++(1 ,0 ) -- ++(0 , -0.25) -- ++(1 ,0 ) -- ++(0 ,1 ) -- ++(-1 ,0 ) -- ++(0 , -0.25 ) -- ++(-1 , 0) -- ++(1 ,0) -- ++(0,0.5) -- ++(-1,0) -- ++(1,0) -- ++(0,-0.25) -- ++(1,0) -- ++(0,-0.5) -- ++(0.5,0) -- ++(0,1) -- ++(-2.5,0) -- ++(2.5,0) -- ++(0,-0.5) -- ++(2,0) -- ++(0,3.25) -- ++(-0.5,0) -- ++(0,-1) -- ++(-1.5,0) -- ++(0,-0.75) -- ++(-1.5,0) -- ++(0,-0.5) -- ++(-1,0) -- ++(1,0) -- ++(0,0.5) -- ++(-1,0) -- ++(1,0) -- ++(0,0.5) -- ++(-1,0) -- ++(1,0) -- ++(0,-0.5) -- ++(1.5,0) -- ++(0,1.5) -- ++(-1.5,0) -- ++(0,-0.5) -- ++(-1,0) -- ++(1,0) -- ++(0,0.5) -- ++(-1,0) -- ++(1,0) -- ++(0,0.5) -- ++(-1,0) -- ++(1,0) -- ++(0,-0.5) -- ++(1.5,0) -- ++(0,-0.75) -- ++(1.5,0) -- ++(0,2) -- ++(-0.5,0) -- ++(0,-0.25) -- ++(-3.5,0) -- ++(3.5,0) -- ++(0,0.5) -- ++(-3.5,0) -- ++(3.5,0) -- ++(0,-0.25) -- ++(0.5,0) -- ++(0,-1) -- ++(0.5,0) -- ++(0,-1.625) -- ++(1,0);

    \node [below] at (12.5,-2.3) {$\delta$};
    \node [below] at (9.5,-2.3) {$\delta_1$};
   
   \draw[fill=green!20, opacity=0.5, rounded corners] (6.25,-2.2) rectangle (9.25,-0.05);

  
  \draw [-stealth] (6.5,-6.5) -- (6.5,-3.5);
  \draw [-stealth] (6.3,-6.3) -- (10.5,-6.3);
  \draw [-, draw=black!30] (9.5,-6.3) -- (9.5,-3.7);
    
 \draw[thick, -] (6.5, -6) -- ++(1,0) -- ++(0 , 0.5 ) -- ++(-1, 0) -- ++(1 ,0 ) -- ++(0 , -0.25) -- ++(1 ,0 ) -- ++(0 ,1 ) -- ++(-1 ,0 ) -- ++(0 , -0.25 ) -- ++(-1 , 0) -- ++(1 ,0) -- ++(0,0.5) -- ++(-1,0) -- ++(1,0) -- ++(0,-0.25) -- ++(1,0) -- ++(0,-0.5) -- ++(0.5,0) -- ++(0,1) -- ++(-2.5,0) -- ++(2.5,0) -- ++(0,-0.5) -- ++(1,0)  ;
   
      \draw[fill=green!20, opacity=0.5, rounded corners] (6.25,-6.2) rectangle (9.25,-4.05);

    \node [below] at (10.5,-6.3) {$\delta$};
    \node [below] at (9.5,-6.3) {$\delta_1$};

    \node [blue vertex, minimum size = 4*\unit] at (1,1) (z) {};
    \path (z) ++ (-2,2) node {\blue{$N$}};
    \node [point, minimum size = 0.1*\unit] at ( 1.4,0.9) (1) {}; \node [point, minimum size = 0.1*\unit] at ( 1, 0) (2) {};
    \node [point, minimum size = 0.1*\unit] at (0.8,-0.5) (3) {}; \node [point, minimum size = 0.1*\unit] at ( 0.5,0.5) (4) {};
    \node [point, minimum size = 0.1*\unit] at (0.3, -0.3) (5) {}; \node [point, minimum size = 0.1*\unit] at ( -0.1,0.2) (6) {};
    \node [point, minimum size = 0.1*\unit] at (1.6, 2) (7) {}; \node [point, minimum size = 0.1*\unit] at ( 0.5, 1.3) (8) {};
    \node [point, minimum size = 0.1*\unit] at ( 0.9,2.5) (9) {}; \node [point, minimum size = 0.1*\unit] at ( 2.4, 1.7) (10) {}; 
    \node [point, minimum size = 0.1*\unit] at (2.2, 0.5) (11) {}; 
    \node [point, minimum size = 0.1*\unit] at (-0.3, 2.3) (12) {}; 
    \node [point, minimum size = 0.1*\unit] at (-0.4, 1.5) (13) {}; 
    \path [stealth-stealth] (5) edge (6)    (3) edge (4)   (2) edge (6) (4) edge (6);
    \path [stealth-stealth] (12) edge (9)    (13) edge (9);
     \path [stealth-stealth] (4) edge (1)  (4) edge (7) (8) edge (7) (8) edge (9) (8) edge (13) (1) edge (10) (2) edge (11)  (10) edge (11);

    \node [green vertex, opacity=0.5, minimum size = 1.6*\unit] at (0.5,0.05) (Nip) {};

    \node [green vertex, minimum size = 1.6*\unit] at (1,-5) (Ni) {};
    \path (Ni) ++ (-1.25,1.25) node {\green{$N_1^{\delta_1}$}};
    \node [point, minimum size = 0.1*\unit] at ( 1.5, -5) (2p) {};
    \node [point, minimum size = 0.1*\unit] at (1.3,-5.5) (3p) {}; \node [point, minimum size = 0.1*\unit] at ( 1,-4.5) (4p) {};
    \node [point, minimum size = 0.1*\unit] at (0.8, -5.3) (5p) {}; \node [point, minimum size = 0.1*\unit] at ( 0.4,-4.8) (6p) {};
           \path [stealth-stealth] (5p) edge (6p)    (3p) edge (4p)   (2p) edge (6p) (4p) edge (6p);

    
       \path [ultra thick, -stealth, color=red, above] (3.25,1) edge node {{$\ccalH$}} (6.25,1);
       \path [ultra thick, -stealth, color=red, above] (2.25,-5) edge node {{$\ccalH$}} (6.25,-5);
       
       \path [thick, -stealth, color=green!20, above, bend left=15] (Nip) edge node {{}} (Ni);
        
       \path [thick, -stealth, dashed, color=red, above, bend left=15] (Nip) edge node {{}} (6.25, -1);

\end{tikzpicture}
}
	\vspace{-0.1in}
	\caption{The clustering method $\ccalH$ is excisive. Given an arbitrary network $N$ (blue) the method $\ccalH$ outputs the dendrogram on the top right, where the green branch corresponds to the subnetwork $N_1^{\delta_1}$. If we consider the isolated subnetwork $N_1^{\delta_1}$ and apply $\ccalH$, excisiveness guarantees that the obtained dendrogram is equivalent to the green branch in the original one.}
	\label{fig_excisiveness_definition}
\end{figure}
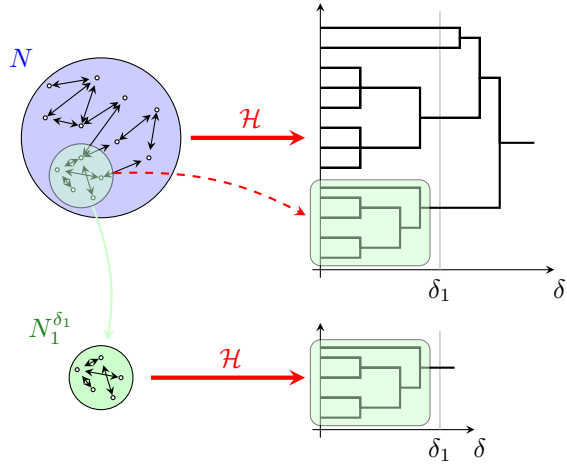
Two related ultrametrics can be defined on the node set represented by any block $B_i(\delta)$. First, the result of restricting the output clustering ultrametric $u_X$ to $B_i(\delta)$. Second, the ultrametric obtained when applying the clustering method $\ccalH$ to the subnetwork $N^\delta_i$. If the two intervening ultrametrics are the same for every network $N$, all $i$, and all $\delta>0$, then we say that the method $\ccalH$ is excisive, as we formally define next.

\begin{indentedparagraph}{(P1) Excisiveness} We say that $\ccalH$ is \emph{excisive} if, for any arbitrary network $N$, for all subnetworks $N^\delta_i$ [cf.~\eqref{eqn_excisiveness_subnetworks}] at all resolutions $\delta > 0$ it holds that
	\begin{equation}\label{eqn_def_excisiveness}
	\ccalH\Big(N^\delta_i\Big) 
	= \left(B_i(\delta),\ u_X \big|_{B_i(\delta) \times B_i(\delta)}\right).
	\end{equation}
\end{indentedparagraph}

\smallskip \noindent The appeal of excisive methods is that they exhibit local consistency in the following sense.
For a given resolution $\delta$, when we cluster the subnetworks as defined in \eqref{eqn_excisiveness_subnetworks}, we obtain a dendrogram on the node set $B_i(\delta)$ for every $i$. Excisiveness ensures that when clustering the whole network and cutting the output dendrogram at resolution $\delta$, the branches obtained coincide with the previously computed dendrograms for every subnetwork; see Fig. \ref{fig_excisiveness_definition}. Our notion of excisiveness is inspired by \cite{CarlssonMemoli10}, where a related concept was analyzed for non-hierarchical clustering of finite metric spaces.

Excisiveness entails a tangible practical advantage when hierarchically clustering big data. 
In applications, one often begins by performing a coarse clustering at an exploratory phase. 
Notice that the computational cost of obtaining this coarse partition, which corresponds to \emph{one} particular resolution, is smaller than that of computing the whole dendrogram. 
After having done this, one focuses on relevant clusters -- via the subsequent application of the clustering method -- in order to reveal the whole hierarchical structure of this subset of the data. 
An excisive method guarantees that the result obtained from this two-step procedure coincides with the more computationally intensive clustering of the whole dataset. 
A specific example of this computational gain is presented next.

\begin{example}[Single linkage computation]\normalfont
	Focus on the application of single linkage hierarchical clustering to a finite metric space of $n$ points. Single linkage is an excisive clustering method as can be concluded by combining Proposition \ref{prop_reciprocal_nonreciprocal_excisive} below with the fact that, for finite metric spaces, reciprocal and nonreciprocal clustering coincide with single linkage (cf.~Theorem~\ref{theo_extremal_ultrametrics}). Consider two different ways of computing the output dendrogram for a subspace of the aforementioned finite metric space. The first approach is to hierarchically cluster the whole finite metric space and then extract the relevant branch. The computational cost of single linkage is equivalent to that of finding a minimum spanning tree in an undirected graph which, for a complete graph, is of cost $\mathcal{O}(n^2)$ \cite[Theorem 1.1]{chazelle_2000_minimum}.\footnote{Notice that the mentioned complexity omits an inverse Ackermann term which is small in practice. Also, randomized algorithms with expected complexity $\mathcal{O}(n^2)$ have been derived; see discussion in \cite{chazelle_2000_minimum}.} 
	The second approach consists of first obtaining the partition given by single linkage corresponding to \emph{one} coarse resolution. This is equivalent to finding the connected components in a graph where only the edges of weight smaller than the resolution are present. Assuming that the average degree of each node in this graph is $\alpha$, the computational cost of finding the connected components is $\mathcal{O}(\max(n, n \, \alpha/2)) =  \mathcal{O}(n \, \alpha/2)$ as long as $\alpha \geq 2$ \cite{HopcroftTarjan73}. After this, we pick the subspace of interest and find its minimum spanning tree. 
	Assuming that the subspace contains $\beta \, n$ nodes, the cost of finding the minimum spanning tree is $\mathcal{O}(\beta^2 n^2)$. 
	Consequently, the cost of the first approach is $\mathcal{O}(n^2)$ whereas the cost of the second one is $\mathcal{O}(n \, \alpha/2) + \mathcal{O}(\beta^2 n^2)$. This entails an asymptotic reduction of order $\beta^{-2}$. In the extreme case where $\beta = \beta_0/n$ so that the subspace of interest is independent of the size of the whole network, there is a reduction in computational complexity from quadratic to linear in $n$. Still, excisiveness ensures that the outputs of both approaches coincide, allowing us to follow the second -- more efficient -- approach.
\end{example}

Having established that excisiveness is a property of practical relevance, we seek to study its relation with the axiomatic approach reviewed in Section~\ref{subsec_admissible_hierarchical_clustering_algorithms}.
One can show that there exist clustering methods that, while satisfying axioms (A1)-(A2), are also excisive. Indeed, the reciprocal and nonreciprocal clustering methods introduced in Section \ref{subsec_admissible_hierarchical_clustering_algorithms} are excisive as we state next.

\begin{proposition}\label{prop_reciprocal_nonreciprocal_excisive}
	The reciprocal $\ccalH^{\R}$ and nonreciprocal $\ccalH^{\NR}$ methods with output ultrametrics defined in \eqref{eqn_reciprocal_clustering} and \eqref{eqn_nonreciprocal_clustering}, respectively, are excisive as defined in (P1).
\end{proposition}
\begin{proof}
    Given an arbitrary network $N=(X, A_X)$, denote by $(X, u_X^\R) =\ccalH^\R(N)$ the output ultrametric when applying $\ccalH^\R$ to $N$. Pick an arbitrary resolution $\delta$ and focus on a subnetwork $N_i^\delta = (X_i^\delta, A_{X_i^\delta})$ as defined in~\eqref{eqn_excisiveness_subnetworks}. Denote by $(X_i^\delta, u^\R_{X_i^\delta})=\ccalH^\R(N_i^\delta)$ the clustering output when applying $\ccalH^\R$ to $N_i^\delta$. We want to show that 
	\begin{align}\label{eqn_proof_reciprocal_excisiveness_010}
		u^\R_{X_i^\delta} \equiv u_X^\R \big|_{X_i^\delta \times X_i^\delta}.
	\end{align}
	Since the network $N$, the resolution $\delta$, and the subnetwork index $i$ were chosen arbitrarily, \eqref{eqn_proof_reciprocal_excisiveness_010} would imply that the reciprocal clustering method $\ccalH^\R$ is excisive [cf. (P1)], as wanted.
	We first show that 
	\begin{align}\label{eqn_proof_reciprocal_excisiveness_020}
		u^\R_{X_i^\delta}(x, x') \geq u_X^\R(x, x'),
	\end{align}
	for all nodes $x, x' \in X_i^\delta$. Notice that the inclusion map $\phi: X_i^\delta \to X$ from network $N_i^\delta$ to $N$ such that $\phi(x)=x$ is a dissimilarity reducing map as defined in (A2). Hence, since $\ccalH^\R$ satisfies the Axiom of Transformation (A2), inequality \eqref{eqn_proof_reciprocal_excisiveness_020} must hold.
	In order to show the opposite inequality, pick arbitrary nodes $x, x' \in X_i^\delta$ and assume that
	\begin{align}\label{eqn_proof_reciprocal_excisiveness_030}
		u^\R_X(x, x') = \alpha.
	\end{align}
	From \eqref{eqn_excisiveness_subnetworks_ultrametrics}, we know that $\alpha \leq \delta$. From the definition of $\ccalH^\R$ in \eqref{eqn_reciprocal_clustering}, equality \eqref{eqn_proof_reciprocal_excisiveness_030} implies that there exists a chain $C(x, x')=[x=x_0, x_1, \ldots, x_l=x']$ where the maximum dissimilarity in both directions between consecutive nodes is $\alpha$. However, notice that part of this chain can be used to join any two nodes $x_j$ and $x_k$ where $j, k \in \{0, 1, \ldots , l\}$ with dissimilarities not larger than $\alpha$. This implies that $u^\R_X(x_j, x_k) = \alpha$ for $j, k \in \{0, 1, \ldots, l\}$ and from the definition of subnetwork [cf. \eqref{eqn_excisiveness_subnetworks_ultrametrics}] we must have that $x_j \in X_i^\delta$ for all $j \in \{0,1, \ldots, l\}$. Consequently, when applying the reciprocal clustering method $\ccalH^\R$ to $N_i^\delta$, the nodes in the chain $C(x, x')$ are contained in its node set $X_i^{\delta}$, allowing us to write [cf. \eqref{eqn_reciprocal_clustering}]
	\begin{align}\label{eqn_proof_reciprocal_excisiveness_050}
		u^\R_{X_i^\delta}(x, x') \leq \!\!  \max_{j | x_j\in C(x,x')}  \!\! \bbarA_{X_i^\delta}(x_j,x_{j+1}) = \alpha = u_X^\R(x, x'),
	\end{align}
	where the inequality comes from the fact that we picked one particular chain $C(x, x')$ instead of minimizing across all possible chains. Since $x, x' \in X_i^\delta$ were picked arbitrarily, \eqref{eqn_proof_reciprocal_excisiveness_050} implies that $u^\R_{X_i^\R}(x, x') \leq u_X^\R(x, x')$ for all $x, x' \in X_i^\delta$. Combining this inequality with \eqref{eqn_proof_reciprocal_excisiveness_020}, equivalence \eqref{eqn_proof_reciprocal_excisiveness_010} follows and we show excisiveness of $\ccalH^\R$.
	A similar argument can be used to show excisiveness of nonreciprocal clustering~$\ccalH^{\NR}$.
\end{proof}

Despite Proposition~\ref{prop_reciprocal_nonreciprocal_excisive}, \emph{excisiveness is not implied by admissibility} with respect to (A1) and (A2). To see this, consider the admissible semi-reciprocal clustering method $\ccalH^{\SR(t)}$ introduced in \cite{Carlssonetal13_2} and briefly explained next.

Semi-reciprocal clustering presents an intermediate behavior between reciprocal and nonreciprocal clustering. In reciprocal clustering, we minimize the cost of a chain in both directions simultaneously whereas in nonreciprocal clustering we minimize the cost in both directions separately. However, semi-reciprocal clustering adopts an intermediate position. In order to formalize this, we denote by $C_t(x,x')$ a chain starting at $x$ and finishing at $x'$ with at most $t$ nodes while we reserve the notation $C(x,x')$ to denote a chain linking $x$ with $x'$ with no maximum imposed on the number of nodes in the chain. Given a network $N=(X, A_X)$, define as $A^{\SR(t)}_X(x, x')$ the minimum cost incurred when traveling from node $x$ to node $x'$ using a chain of at most $t$ nodes. I.e.,
\begin{equation}\label{eqn_inter_cost}
A^{\SR(t)}_X(x, x') := \min_{C_t(x, x')} \,\,\,  \max_{i | x_i\in C_t(x, x')} A_X(x_i, x_{i+1}).
\end{equation}
The family of semi-reciprocal clustering methods $\ccalH^{\SR(t)}$ with output $(X,u^{\SR(t)}_X)=\ccalH^{\SR(t)}(N)$ is defined as 
\begin{align}\label{eqn_inter_reciprocal_clustering} 
u^{\SR(t)}_X(x,x') := \min_{C(x,x')} \,\,\, \max_{i | x_i\in C(x,x')} \bar{A}^{\SR(t)}_X(x_i, x_{i+1}),
\end{align}     
where the function $\bar{A}^{\SR(t)}_X(x_i, x_{i+1})$ is computed as follows
\begin{align}\label{eqn_inter_reciprocal_clustering_auxiliary}     
\bar{A}^{\SR(t)}_X(x_i, x_{i+1}) :=
\max \big(A^{\SR(t)}_X(x_i, x_{i+1}), A^{\SR(t)}_X(x_{i+1}, x_i)\big).
\end{align} 

We can interpret \eqref{eqn_inter_reciprocal_clustering} as the application of reciprocal clustering [cf. \eqref{eqn_reciprocal_clustering}] to a network with dissimilarities given by $A^{\SR(t)}_X$ in \eqref{eqn_inter_cost}, i.e., a network with dissimilarities given by the optimal choice of chains of constrained length $t$. Semi-reciprocal clustering methods satisfy axioms (A1)-(A2); see \cite[Prop. 4]{Carlssonetal13_2}.

To see that admissibility does not imply excisiveness, consider the network in Fig. \ref{fig_excisiveness_counter_example} and its dendrogram corresponding to the semi-reciprocal clustering method $\ccalH^{\SR(3)}$. For a resolution $\delta=1.5$, focus on the subnetwork $N^{1.5}_1=(\{x_1, x_3\}, A_{\{1,3\}})$ with $A_{\{1,3\}}(x_1, x_3)=A_{\{1,3\}}(x_3, x_1)=2$. When the clustering method $\ccalH^{\SR(3)}$ is applied to this subnetwork, the output dendrogram (red) differs from the corresponding branch in the original dendrogram (green). This counterexample shows that excisiveness cannot be derived from axioms (A1) and (A2).

%
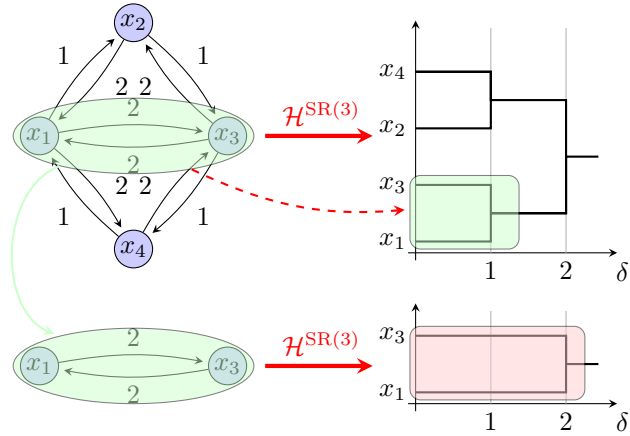
\begin{figure}
	\centering \def \thisplotscale {0.5}
\def \unit {\thisplotscale cm}

{\small
\begin{tikzpicture}[-stealth, shorten >=2, scale = \thisplotscale]


    \node [blue vertex] at (-1,3.8) (1) {$x_2$};
    \node [blue vertex] at (1.5,0.8) (2) {$x_3$};    
    \node [blue vertex] at (-3.5,0.8) (3) {$x_1$};
    \node [blue vertex] at (-1,-2.2) (4) {$x_4$};

    \path (1) edge [bend left=10, above right] node {$1$} (2);	
    \path (2) edge [bend left=10, below] node {$2$} (3);
    \path (3) edge [bend left=10, above left] node {$1$} (1);    	

    \path (2) edge [bend left=10, below left, pos=0.7] node {$2$} (1);	
    \path (3) edge [bend left=10, above] node {$2$} (2);
    \path (1) edge [bend left=10, below right, pos=0.3]  node {$2$} (3);    	
    
    \path (3) edge [bend left=10, above right, pos=0.7] node {$2$} (4);    	
    \path (4) edge [bend left=10, below left] node {$1$} (3);	
    \path (4) edge [bend left=10, above left, pos=0.3] node {$2$} (2);
    \path (2) edge [bend left=10, below right]  node {$1$} (4);  
    
   
   \draw [-stealth] (6.5,-2.5) -- (6.5,4);
    \draw [-stealth] (6.3,-2.3) -- (12,-2.3);
    \draw [-, draw=black!30] (8.5,-2.3) -- (8.5,3.8);
      \draw [-, draw=black!30] (10.5,-2.3) -- (10.5,3.8);
    
    \draw[thick, -] (6.5, -2) -- ++(2,0) -- ++(0,1.5) -- ++(-2,0) ++(0,1.5) -- ++(2,0) -- ++(0,1.5) -- ++(-2,0) -- ++(2,0) -- ++(0,-0.75) -- ++(2,0) -- ++(0,-3) -- ++(-2,0) -- ++(2,0)-- ++(0,1.5)-- ++(1,0);

    \node [below] at (12,-2.3) {$\delta$};
    \node [below] at (8.5,-2.3) {$1$};
    \node [below] at (10.5,-2.3) {$2$};
    \node [left] at (6.5,-2) {$x_1$};
    \node [left] at (6.5,-0.5) {$x_3$};
    \node [left] at (6.5,1) {$x_2$};
    \node [left] at (6.5,2.5) {$x_4$};
 
       
     \path [ultra thick, -stealth, color=red, above] (2.5,0.8) edge node {{$\ccalH^{\SR(3)}$}} (5.5,0.8);
  
     \draw[fill=green!20, opacity=0.5, rounded corners] (6.35,-2.2) rectangle (9.25, -0.25);
     \node at (-1, 0.8) [ellipse, minimum height=1cm,minimum width=3.2cm, fill=green!20, opacity=0.5, draw] (sn) {};
       \path [thick, -stealth, dashed, color=red, above, bend right=15] (sn) edge node {{}} (6.35, -1.15);

       
    \node [blue vertex] at (1.5,-5.3) (2p) {$x_3$};    
    \node [blue vertex] at (-3.5,-5.3) (3p) {$x_1$};

    \path (2p) edge [bend left=10, below] node {$2$} (3p);

    \path (3p) edge [bend left=10, above] node {$2$} (2p);
    
    
   \draw [-stealth] (6.5,-6.5) -- (6.5,-3.5);
    \draw [-stealth] (6.3,-6.3) -- (12,-6.3);
    \draw [-, draw=black!30] (8.5,-6.3) -- (8.5,-3.7);
      \draw [-, draw=black!30] (10.5,-6.3) -- (10.5,-3.7);
    
    \draw[thick, -] (6.5, -6) -- ++(4,0) -- ++(0,1.5) -- ++(-4,0) ++(4,0) -- ++(0,-0.75) -- ++(1,0);

    \node [below] at (12,-6.3) {$\delta$};
    \node [below] at (8.5,-6.3) {$1$};
    \node [below] at (10.5,-6.3) {$2$};
    \node [left] at (6.5,-6) {$x_1$};
    \node [left] at (6.5,-4.5) {$x_3$};
    
         \path [ultra thick, -stealth, color=red, above] (2.5,-5.3) edge node {{$\ccalH^{\SR(3)}$}} (5.5,-5.3);
    
      \draw[fill=red!20, opacity=0.5, rounded corners] (6.35,-6.2) rectangle (11, -4.25);
     \node at (-1, -5.3) [ellipse, minimum height=1cm,minimum width=3.2cm, fill=green!20, opacity=0.5, draw] (sn2) {};
       \path [thick, -stealth, color=green!20, above, bend right=68] (sn) edge node {{}} (sn2);

\end{tikzpicture}
}
	\vspace{-4mm}
	\caption{Admissibility does not imply excisiveness. The admissible method $\ccalH^{\SR(3)}$ does not satisfy the excisiveness condition since the green branch in the top dendrogram differs from the red branch in the lower one [cf. Fig. \ref{fig_excisiveness_definition}].}
	\label{fig_excisiveness_counter_example}
\end{figure}

\subsection{Linear scale preservation}\label{Ss:linear_scale_preservation}

Consider a network $N_X=(X,A_X)$ and a linear function  $\psi:\reals_+\to\reals_+$ where $\psi(z)=\alpha \, z$, for some $\alpha  > 0$. Define the network $N_X^\psi:=(X, \psi \circ A_X)$ with the same set of nodes and linearly scaled dissimilarities.
With this notation in place, we formally define our second robustness property.

\begin{indentedparagraph}{(P2) Linear Scale Preservation}  We say that $\ccalH$ is \emph{linear scale preserving} if for any arbitrary network $N_X$ and function $\psi$ satisfying the above requirements, the outputs $(X,u_X):=\ccalH(N_X)$ and $(X, u_X^\psi):=\ccalH(N_X^\psi)$ satisfy
	\begin{equation}\label{eqn_linear_scale_preserving}
	u_X^\psi = \psi \circ u_X.
	\end{equation}
\end{indentedparagraph}

\smallskip\noindent For linear scale preserving methods, the ultrametric outcomes vary according to the same linear function that transforms the dissimilarity function. Hence, the hierarchical structure output by these methods is invariant with respect to units. In terms of dendrograms, linear scale preservation entails that a transformation of dissimilarities with an appropriate linear function $\psi$ results in a dendrogram where the order in which nodes are clustered together is the same while the resolution at which mergings occur changes linearly according to~$\psi$.

In practice, linear scale preservation is a desirable property. E.g., if we want to hierarchically cluster finite metric spaces -- which are particular cases of asymmetric networks where every dissimilarity is symmetric and the triangle inequality is satisfied -- the hierarchy of the output should not depend on the unit used to measure distances. Equivalently, the choice of units does not alter the nature of a given metric space, thus, if we measure distances in, e.g., meters or centimeters we should obtain the same structure when clustering both. Linear scale preserving methods guarantee this behavior for arbitrary asymmetric networks.

The reciprocal and nonreciprocal clustering methods introduced in Section~\ref{subsec_admissible_hierarchical_clustering_algorithms} are linear scale preserving.

\begin{proposition}\label{prop_reciprocal_nonreciprocal_linear_scale_preserving}
	The reciprocal $\ccalH^{\R}$ and nonreciprocal $\ccalH^{\NR}$ methods with output ultrametrics defined in \eqref{eqn_reciprocal_clustering} and \eqref{eqn_nonreciprocal_clustering}, respectively, are linear scale preserving as defined in (P2). 
\end{proposition}

\begin{proof}
    This proposition follows as a particular case of our main result (Theorem~\ref{theo_representability_excisiveness}, to be introduced in Section~\ref{subsec_generative_model}) after demonstrating that $\ccalH^\R$ and $\ccalH^\NR$ are representable methods; cf. Fig.~\ref{fig_reciprocal_nonreciprocal_representable} and associated discussion. Notice that Proposition~\ref{prop_reciprocal_nonreciprocal_excisive} can also be shown as a particular case of the more general Theorem~\ref{theo_representability_excisiveness}, but we decided to include its proof to demonstrate the technique to show one of these simpler results independently. Nonetheless, the current proof is omitted to avoid redundancy.
\end{proof}

Proposition \ref{prop_reciprocal_nonreciprocal_linear_scale_preserving} notwithstanding, linear scale preservation is a condition independent of axioms (A1) and (A2). This can be seen by analyzing the behavior of the admissible method $\ccalH^{\R/\NR}(\beta)$ introduced in \cite{Carlssonetal13_2} and briefly explained next.

The grafting clustering method $\ccalH^{\R/\NR}(\beta)$ is constructed by pasting branches of the nonreciprocal dendrogram into corresponding branches of the reciprocal dendrogram. To define this precisely, one computes the reciprocal and nonreciprocal dendrograms and cut all branches of the reciprocal dendrogram at resolution $\beta>0$. Then, replace the cut branches of the reciprocal tree by the corresponding branches -- i.e., those with the same leaves -- of the nonreciprocal tree. This hybrid dendrogram is the output of method $\ccalH^{\R/\NR}(\beta)$. In terms of ultrametrics, we can define this pasting formally as follows.
\begin{equation}\label{def_mu_beta_1}
u^{\R/\NR}_X(x,x';\beta) :=
\begin{cases}
u^{\NR}_X(x,x'), & \text{if } u^{\R}_X(x,x') \leq \beta, \\
u^{\R}_X(x,x'),  & \text{if } u^{\R}_X(x,x') >    \beta .
\end{cases}
\end{equation}
The ultrametric in \eqref{def_mu_beta_1} is valid and $\ccalH^{\R/\NR}(\beta)$ satisfies axioms (A1) and (A2); see \cite[Prop. 1]{Carlssonetal13_2}. 
However, the method $\ccalH^{\R/\NR}(\beta)$ is not linear scale preserving as can be seen from a simple counterexample. Consider the three-node network in Fig.~\ref{fig_scale_invariance_counter} as well as its transformation after applying the linear function $\psi(z)=2\,z$. The figure illustrates the fact that the reciprocal and nonreciprocal ultrametrics are transformed by $\psi$, as it should be given Proposition \ref{prop_reciprocal_nonreciprocal_linear_scale_preserving}. However, we see that the ultrametric output by $\ccalH^{\R/\NR}(\beta)$ (for $\beta=3)$ is multiplied by 4 instead of by 2, thus violating (P2).

%
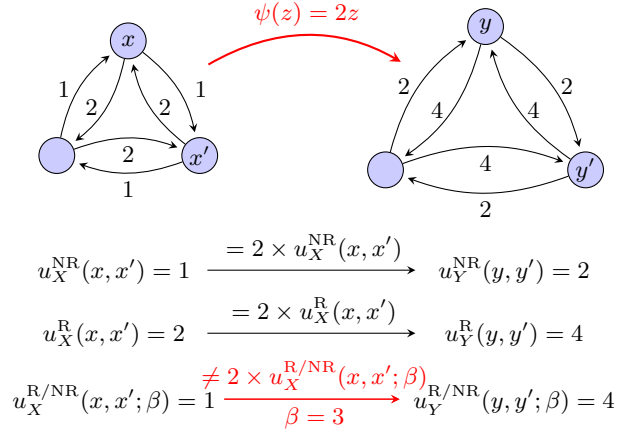
\begin{figure}
	\centering
	\def \thisplotscale {0.38}
\def \unit {\thisplotscale cm}

{\footnotesize
\begin{tikzpicture}[-stealth, shorten >=2, x = 1*\unit, y=1*\unit]

    \node [blue vertex, minimum size=1.25*\unit] at (12.0,5.0) (1) {$y$};
    \node [blue vertex, minimum size=1.25*\unit] at (15.5,  0) (2) {$y'$};    
    \node [blue vertex, minimum size=1.25*\unit] at (8.5,  0)  (3) {$$};
    
    \node [fill=white] at (9.5,  3.5) (4) {};

    \path (1) edge [bend left=20, right] node {{$2$}} (2);	
    \path (2) edge [bend left=20, below] node {{$2$}} (3);
    \path (3) edge [bend left=20, left] node  {{$2$}} (1);    	

    \path (2) edge [bend left=20, right] node {{$4$}} (1);	
    \path (3) edge [bend left=20, below] node {{$4$}} (2);
    \path (1) edge [bend left=20, left]  node {{$4$}} (3);    	
    
    \node [blue vertex, minimum size=1.25*\unit] at (-0.5,4.5) (1p) {$x$};
    \node [blue vertex, minimum size=1.25*\unit] at (2,0.5) (2p) {$x'$};    
    \node [blue vertex, minimum size=1.25*\unit] at ( -3,0.5) (3p) {$$};
    
    \node [fill=white] at ( 2,3.5) (4p) {};

    \path (1p) edge [bend left=20, right] node {{ $1$}} (2p);	
    \path (2p) edge [bend left=20, below] node {{$1$}} (3p);
    \path (3p) edge [bend left=20, left]  node {{$1$}} (1p);    	

    \path (2p) edge [bend left=20, right] node {{$2$}} (1p);	
    \path (3p) edge [bend left=20, below] node {{$2$}} (2p);
    \path (1p) edge [bend left=20, left]  node {{$2$}} (3p);    	
    
    \path (4p) edge [above, red, thick, pos=0.5, bend left] node {$\psi (z) = 2z$} (4);

    \path     (-1, -3.5)  node [minimum width = 6.5*\unit] (nrx) {$u_X^{\NR}(x,x')=1$}
           ++ (14,0)    node [minimum width = 6.5*\unit] (nry) {$u_Y^{\NR}(y,y')=2$};
    \path (nrx) edge [above] node {$= 2\times u_X^{\NR}(x,x')$} (nry);

    \path     (-1, -5.7) node (rx) [minimum width = 6.5*\unit] {$u_X^{\R}(x,x')=2$}
           ++ (14,0)   node (ry) [minimum width = 6.5*\unit]{$u_Y^{\R}(y,y')=4$};
    \path (rx) edge [above] node {$= 2\times u_X^{\R}(x,x')$} (ry);

    \path     (-1, -8.0) node (rx) [minimum width = 6.5*\unit] {$u_X^{\R/\NR}(x,x';\beta)=1$}
           ++ (14,0)   node (ry) [minimum width = 6.5*\unit]{$u_Y^{\R/\NR}(y,y';\beta)=4$};
    \path (rx) edge [above,red] node {$\neq 2\times u_X^{\R/\NR}(x,x';\beta)$} (ry);
    \path (rx) edge [below, red] node {$\beta=3$}(ry);


\end{tikzpicture}
}
	\caption{Admissibility does not imply linear scale preservation. Reciprocal and nonreciprocal clustering are linear scale preserving while $\ccalH^{\R/\NR}(\beta)$ is not.}
	\label{fig_scale_invariance_counter}
\end{figure}

\subsection{Stability}\label{Ss:stability} 

As a third important robustness property, we introduce the notion of stability.
This concept requires the definition of a metric $d_\ccalN$ between networks. 
This metric is a generalization of the Gromov-Hausdorff distance \cite[Chapter 7.3]{burago-book}, originally conceived as a metric between compact metric spaces, to the more general collection of networks $\ccalN$.

Whenever two networks $N_X$ and $N_Y$ are related by a simple redefinition of the node labels, we say that they are isomorphic and we write $N_X \cong N_Y$. 
The collection of networks where all isomorphic networks are represented by a single point is called the collection of networks modulo isomorphism and denoted as $\ccalN\mod\cong$.
For node sets $X$ and $Y$ consider subsets $R\subseteq X\times Y$ of the Cartesian product set $X\times Y$ with elements $(x,y)\in R$. The set $R$ is a \emph{correspondence} between $X$ and $Y$ if for all $x_0\in X$ we have at least one element $(x_0,y)\in R$ and for all $y_0\in Y$ we have at least one element $(x,y_0)\in R$. 
The metric $d_\ccalN$ between networks $N_X$ and $N_Y$ takes the value
\begin{align}\label{eqn_gh_distance}
	d_\ccalN(N_X,N_Y) := &\frac{1}{2}\min_{R}\!\max_{(x,y),(x',y')\in R} \big|A_X(x,x')-A_Y(y,y')\big|.
\end{align}
Definition \eqref{eqn_gh_distance} is a verbatim generalization of the Gromov-Hausdorff distance in \cite[Theorem 7.3.25]{burago-book} except that the dissimilarity functions $A_X$ and $A_Y$ are not restricted to be metrics. For this more general case, $d_\ccalN$ is still a legitimate metric in the space $\ccalN\mod\cong$; {see \cite[Supp. mats. A.4]{Carlssonetal14} for a proof of this fact. See \cite{pph} for applications to the stability of persistent homology over networks}. {The case of possibly \emph{infinite} networks was studied in \cite{cm17,cm18,chowdhury2017convergence}.\footnote{{These papers consider a notion of network $(X,A_X)$ a bit more general than the one we considered here: the authors do not require $A_X$ to satisfy any conditions except being real valued.}}}
With this definition in place, we formally introduce the property of stability.

\begin{indentedparagraph}{(P3) Stability} We say that $\ccalH$ is \emph{stable} if there exists a \emph{finite} constant $L=L(\mathcal{H})\geq 0$ such that, for any two networks $N_X$ and $N_Y$, it holds
	\begin{equation}
	d_{\mathcal{N}}\big(\mathcal{H}(N_X),\mathcal{H}(N_Y)\big)\leq L\cdot d_{\mathcal{N}}(N_X,N_Y).
	\end{equation}
\end{indentedparagraph}

\smallskip \noindent Stability ensures that small perturbations on a network result in small perturbations in the associated ultrametric. 
More precisely, perturbations of size at most $\varepsilon$  on a given network -- as measured by $d_{\mathcal{N}}$ -- result in perturbations in the clustering results which are bounded by $L\, \varepsilon$. 
In other words, every stable hierarchical clustering method $\mathcal{H}$ is \emph{Lipschitz} as a map from $(\mathcal{N},d_{\mathcal{N}})$ into itself, making them suitable for practical applications. {See \cite[Remark 17]{clust-um} for a discussion of the fact that average and complete linkage hierarchical methods are unstable when applied to finite metric spaces.}

Mimicking the developments in previous subsections, one can show that the reciprocal and non-reciprocal clustering methods are stable.

\begin{proposition}\label{prop_reciprocal_nonreciprocal_stable}
	The reciprocal $\ccalH^{\R}$ and nonreciprocal $\ccalH^{\NR}$ methods with output ultrametrics defined in \eqref{eqn_reciprocal_clustering} and \eqref{eqn_nonreciprocal_clustering}, respectively, are stable as defined in (P3). 
\end{proposition}
\begin{proof}
    This proposition follows as a particular case of our main result (Theorem~\ref{theo_representability_excisiveness}, to be introduced in Section~\ref{subsec_generative_model}) after demonstrating that $\ccalH^\R$ and $\ccalH^\NR$ are representable methods; cf. Fig.~\ref{fig_reciprocal_nonreciprocal_representable} and associated discussion.
\end{proof}

We say that a clustering methods is \emph{robust if it satisfies the properties of excisiveness~(P1), linear scale preservation~(P2), and stability~(P3)}.
Given that robustness is an important practical feature, we want to characterize the family of robust admissible methods. 
From Propositions \ref{prop_reciprocal_nonreciprocal_excisive}, \ref{prop_reciprocal_nonreciprocal_linear_scale_preserving}, and~\ref{prop_reciprocal_nonreciprocal_stable} we know that reciprocal and nonreciprocal clustering belong to this family. 
Our objective is to find if other methods are contained within this family and, more importantly, to provide a comprehensive description of these. 
To this end, we introduce the concept of representability next.

\section{Representability}\label{sec_representability}

We define a representable hierarchical clustering method as one where the clustering of arbitrary networks is specified through the clustering of particular examples that we call \emph{representers}. 
Representers are {possibly} asymmetric networks $\omega=(X_\omega,A_\omega)$ with the distinction that the dissimilarity function $A_\omega$ need not be defined for all pairs of nodes, i.e., $\dom(A_\omega)\neq X_\omega\times X_\omega$.
In this sense, representers are more general objects than networks as introduced in Section~\ref{sec_preliminaries}.

Given an arbitrary network $N=(X, A_X)$, and a representer $\omega=(X_\omega, A_\omega)$, we define the \emph{expansion constant} of a map $\phi: X_\omega \to X$ from $\omega$ to $N$ as
\begin{equation}\label{eqn:def_lipschitz_constant}
L(\phi;\omega, N) := \max_{\substack{(z, z') \in \text{dom}(A_\omega) \\ z \neq z'}} \frac{A_X(\phi(z), \phi(z'))}{A_\omega(z, z')}.
\end{equation}
Notice that $L(\phi;\omega, N)$ is the minimum multiple of $\omega$ such that the map $\phi$ is dissimilarity reducing  as defined in (A2) from $L(\phi;\omega, N) * \omega$ to $N$. Notice as well that the maximum in \eqref{eqn:def_lipschitz_constant} is computed for pairs $(z, z')$ in the domain of $A_\omega$. Pairs not belonging to the domain can be mapped to any dissimilarity without modifying the value of the expansion constant. We define the optimal multiple $\lambda_X^{\omega}(x, x')$ between $x$ and $x'$ in $X$ with respect to $\omega$ as
\begin{equation}\label{eqn:multiple_as_lipschitz_constant}
\lambda_X^\omega(x, x') \! := \! \min \big\{ L(\phi;\omega, N) \,\, | \,\, \phi:X_\omega \to X, \,\, x, x' \in \text{Im}(\phi) \! \big\}.
\end{equation}
Equivalently, $\lambda_X^\omega(x, x')$ is the minimum expansion constant among those maps that have $x$ and $x'$ in their image. I.e., it is the minimum multiple needed for the existence of a dissimilarity reducing map from a multiple of $\omega$ to $N$ that has $x$ and $x'$ in its image.

We can now define the representable method $\ccalH^{\omega}$ associated with a given representer $\omega$ by defining the cost of a chain $C(x,x')=[x=x_0,\ldots, x_l=x']$ linking $x$ to $x'$ as the maximum optimal multiple $\lambda^{\omega}_X(x_i, x_{i+1})$ between consecutive nodes in the chain. The ultrametric $u^{\omega}_X$ associated with output $(X,u^{\omega}_X)=\ccalH^{\omega}(X,A_X)$ is given by the minimum chain cost
\begin{equation}\label{eqn_def_representability_single_network}
u^{\omega}_X(x, x') := 
\min_{C(x, x')} \ \max_{i | x_i \in C(x, x')} \lambda^{\omega}_X(x_i, x_{i+1}),
\end{equation}
for all $x, x' \in X$.
Representable methods are generalized to cases in which we are given a nonempty collection $\Omega$ of representers $\omega$. In such case, we define the function $\lambda^{\Omega}_X$ as
\begin{equation}\label{eqn_Omega_multiple_from_omega}
\lambda^{\Omega}_X(x, x') \ 
:=\ \inf_{\omega \in \Omega}\ \lambda^{\omega}_X(x, x'),
\end{equation}
for all $x, x' \in X$.
The value $\lambda^{\Omega}_X(x, x')$ is the infimum across all optimal multiples given by the different representers $\omega \in \Omega$. For a given network $N=(X, A_X)$, the representable clustering method $\ccalH^\Omega$ associated with the collection of representers $\Omega$ is the one with outputs $(X,u^{\Omega}_X)=\ccalH^{\Omega}(X,A_X)$ such that the ultrametric $u^{\Omega}_X$ is given by
\begin{equation}\label{eqn_def_represent_1}
u^{\Omega}_X(x, x')
:= \min_{C(x, x')}\ \max_{i | x_i \in C(x, x')} \lambda^{\Omega}_X(x_i, x_{i+1}),
\end{equation}
for all $x, x' \in X$. See Fig.~\ref{fig_representability} for an illustrative example.

%
\begin{figure}
	\centering
	\includegraphics[width=0.75\linewidth]{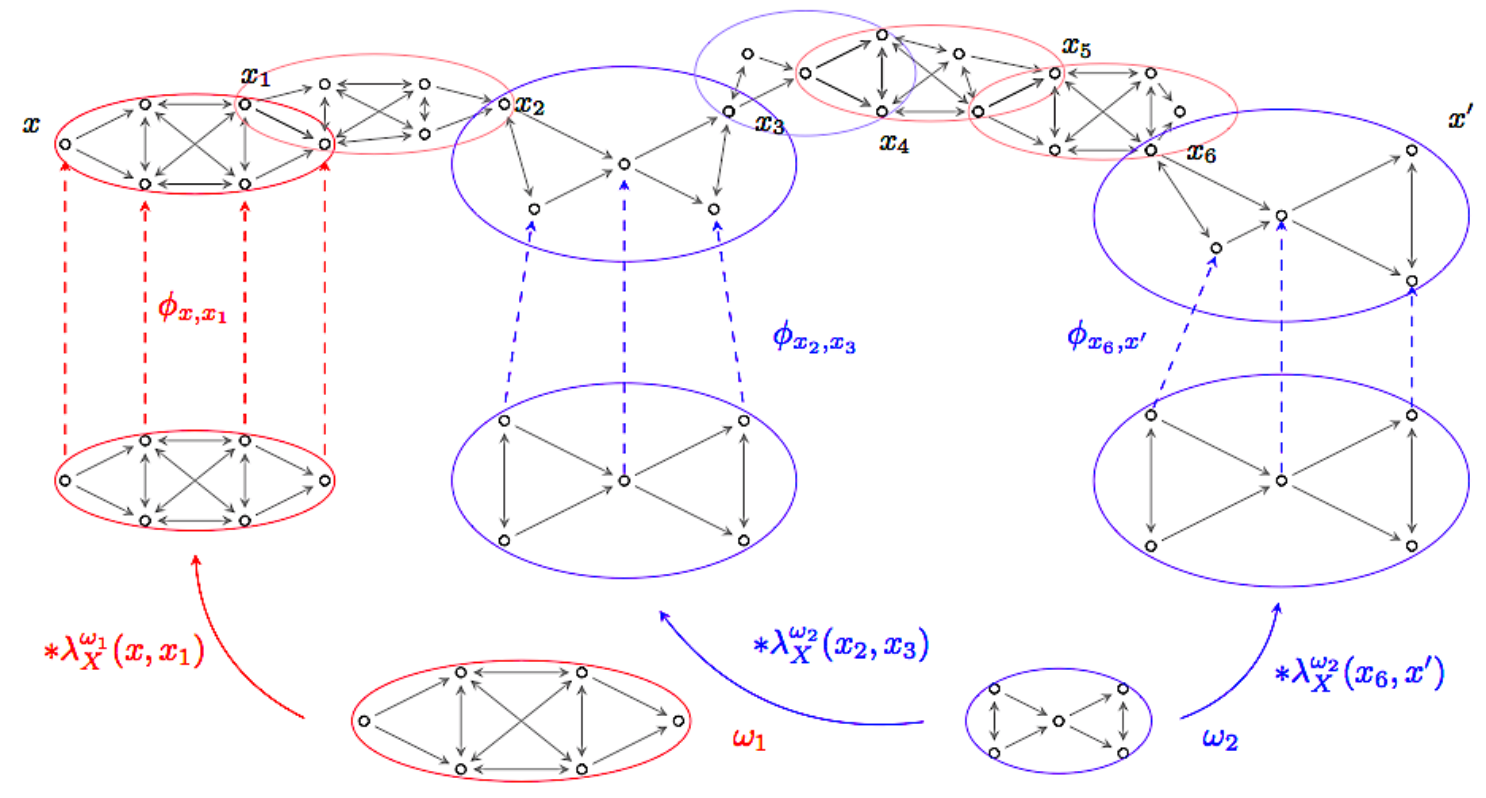}
	\caption{{Representable method $\ccalH^{\Omega}$ with ultrametric output as in~\eqref{eqn_def_represent_1}. The collection of representers $\Omega = \{\omega_1, \omega_2\}$ is shown at the bottom. In order to compute $u^{\Omega}_X(x, x')$ we link $x$ and $x'$ through a chain, e.g. $[x, x_1, \ldots, x_6, x']$ in the figure, and link pairs of consecutive nodes with multiples of the representers. The ultrametric value $u^{\Omega}_X(x, x')$ is given by minimizing over all paths joining $x$ and $x'$ the maximum multiple of a representer used to link consecutive nodes in the path [cf.~\eqref{eqn_def_represent_1}].}}
	\label{fig_representability}
\end{figure}

As we mentioned, not all dissimilarities are necessarily defined in representers. However, the issue of whether a representer is connected or not plays a prominent role in the validity and admissibility of representable methods. We say that a representer $\omega=(X_\omega, A_\omega)$ is \emph{weakly connected} if for every pair of nodes $z, z' \in X_\omega$ we can find a chain $C(z, z')=[z=z_0,\ldots, z_{l}=z']$ such that either $(z_i, z_{i+1}) \in \text{dom}(A_\omega)$ or $(z_{i+1},z_i) \in \text{dom}(A_\omega)$ or both for all $i=0,\ldots,l-1$. Moreover, we say that $\Omega$ is \emph{uniformly bounded} if and only if there exists a finite $M>0$ such that
\begin{equation}\label{eqn_def_uniform_bounded}
\max_{(z, z') \in \text{dom}(A_\omega)} A_\omega(z, z') \leq M,
\end{equation}
for all $\omega \in \Omega$.
For any representer $\omega$, let $\mathrm{sep}(\omega):= \min_{\substack{(z, z') \in \text{dom}(A_\omega)}} A_\omega(z,z')$ and, for a family $\Omega$ of representers, we define $\mathrm{sep}(\Omega):=\inf_{\omega\in\Omega}\mathrm{sep}(\omega)$.
We can now formally define the notion of representability.

\begin{indentedparagraph}{(P4) Representability} We say that a clustering method $\ccalH$ is \emph{representable} if there exists a uniformly bounded collection $\Omega$ of weakly connected representers each with a finite number of nodes and $\mathrm{sep}(\Omega)>0$ such that $\ccalH \equiv \ccalH^\Omega$ where $\ccalH^\Omega$ has output ultrametrics as in \eqref{eqn_def_represent_1}.\end{indentedparagraph}

\smallskip\noindent It can be shown that indeed under the conditions in (P4), \eqref{eqn_def_represent_1} defines a valid ultrametric, as stated next.\footnote{Longer proofs such as the one associated with this result have been deferred to the appendix.}

\begin{proposition}\label{prop_validity_repre}
	For every collection of representers $\Omega$ satisfying the conditions in (P4), \eqref{eqn_def_represent_1} defines a valid ultrametric.
\end{proposition}

Representability allows the definition of universal hierarchical clustering methods from given representative examples. Every representer $\omega \in \Omega$ can be understood as defining a specific structure that can be considered as a cluster unit. The scaling of this cluster unit [cf. \eqref{eqn:multiple_as_lipschitz_constant}] and its replication throughout the network [cf. \eqref{eqn_def_representability_single_network}] signal the resolution at which nodes become part of the same cluster. For nodes $x$ and $x'$ to cluster together at resolution $\delta$, we need to construct a path from $x$ to $x'$ with overlapping versions of representers scaled by parameters not larger than $\delta$. When we have multiple representers, we can use any of them to build these chains [cf. \eqref{eqn_Omega_multiple_from_omega} and \eqref{eqn_def_represent_1}].

Although seemingly unrelated, the property of representability (P4) is tightly related to the more practical requirements of excisiveness (P1), linear scale preservation (P2), and stability (P3), as will be formally shown in Section~\ref{subsec_generative_model}.

\subsection{Factorization of representable methods}\label{sec_representability_and_single_linkage}

The following factorization property for representable methods has practical value in itself and will be instrumental to show our main result in Theorem~\ref{theo_representability_excisiveness}.
Every representable clustering method factors into the composition of two maps: a symmetrizing map that depends on $\Omega$ followed by single linkage hierarchical clustering. This is formally stated next.

%
\begin{proposition}\label{prop_equivalence_representable_single_linkage}
	Every representable clustering method $\ccalH^\Omega$ admits a decomposition of the form   $\ccalH^\Omega \equiv \ccalH^{\SL} \circ \Lambda^\Omega$, where $\Lambda^\Omega: \ccalN \to \ccalN^\mathrm{sym}$ is a map from the collection of asymmetric networks $\ccalN$ to that of symmetric networks $\ccalN^\mathrm{sym}$ and $\ccalH^{\SL}:\ccalN^\mathrm{sym} \to \ccalU$ is the single linkage clustering method for symmetric networks.
\end{proposition}
\begin{proof}
    The proof is just a matter of identifying elements in \eqref{eqn_def_represent_1}. Define the function $\Lambda^\Omega$ as the one that maps the network $N=(X, A_X)$ into $\Lambda^\Omega(X,A_X) = (X, \lambda^{\Omega}_X)$, where the dissimilarity function $\lambda^{\Omega}_X$ has values given by \eqref{eqn_Omega_multiple_from_omega}. That $(X, \lambda^{\Omega}_X)$ is a symmetric network -- i.e., that $\lambda^{\Omega}_X$ satisfies symmetry and identity -- is shown in the proof of Proposition~\ref{prop_validity_repre}. Comparing the definitions of the output ultrametrics of the representable method $\ccalH^\Omega$ in \eqref{eqn_def_represent_1} and of single linkage method in Section~\ref{sec_preliminaries}, we conclude that
\begin{equation}\label{eqn_equivalence_representable_single_linkage}
\ccalH^\Omega(X, A_X)
=\ccalH^{\SL}(X, \lambda^{\Omega}_X)
= \ccalH^{\SL} \big( \Lambda^\Omega (X, A_X) \big),
\end{equation}
as wanted. 
\end{proof}

%
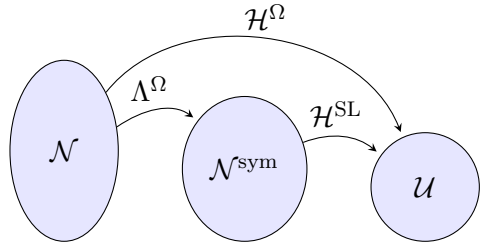
\begin{figure}
	\centering
	\def \thisplotscale {0.48}
\def \unit {\thisplotscale cm}

\begin{tikzpicture}[-stealth, shorten >=2, x = 1*\unit, y=1*\unit]

    \path (0,0) node [draw, fill = blue!10, ellipse, 
                     minimum width = 3*\unit, 
                     minimum height = 5*\unit] (asymmetric networks) {$\ccalN$}
                  ;
    \path (5,-0.5) node [draw, fill = blue!10, ellipse, 
                     minimum width = 3*\unit, 
                     minimum height = 4*\unit] (symmetric networks) {$\ccalN^\mathrm{sym}$};
    \path (10,-1) node [draw, fill = blue!10, ellipse, 
                     minimum width = 3*\unit, 
                     minimum height = 3*\unit] (ultrametrics) {$\ccalU$};

    \path[-stealth] (asymmetric networks)
                    edge [bend left, above, 
                    pos=0.4, 
                    out=30, 
                    in = 140] node {$\Lambda^\Omega$} 
                    (symmetric networks);
    \path[-stealth] (symmetric networks)
                    edge [bend left, above, 
                    pos=0.4, 
                    out=30, 
                    in = 150] node {$\ccalH^{\SL}$} 
                    (ultrametrics); 
    \path[-stealth] (asymmetric networks)
                    edge [bend left, above, 
                    pos=0.5, 
                    out=60, 
                    in = 120] node {$\ccalH^{\Omega}$} 
                    (ultrametrics);

\end{tikzpicture}
	\caption{{Decomposition of representable methods. A representable method can be decomposed into a map from the collection of asymmetric networks to the collection of symmetric networks composed with the single linkage map into the collection of ultrametrics. See Proposition \ref{prop_equivalence_representable_single_linkage}.}}
	\label{fig_decomposition_of_representable_methods}
\end{figure}

Representable clustering methods, as all other hierarchical clustering methods, are maps from the collection of asymmetric networks $\ccalN$ to the collection of ultrametrics $\ccalU$; see Fig. \ref{fig_decomposition_of_representable_methods}. Proposition \ref{prop_equivalence_representable_single_linkage} allows the decomposition of these maps into two components with definite separate roles. The first element of the composition is the function $\Lambda^\Omega$ whose objective is to symmetrize the original, possibly asymmetric, dissimilarity function. This transformation is followed by an application of single linkage $\ccalH^{\SL}$ with the goal of inducing an ultrametric structure on this symmetric, but not necessarily ultrametric, intermediate network. Proposition \ref{prop_equivalence_representable_single_linkage} attests that there may be many different ways of inducing a symmetric structure depending on the selection of the representers in $\Omega$ but that there is a unique method to induce ultrametric structure. This unique method is single linkage hierarchical clustering.  

From an algorithmic perspective, Proposition~\ref{prop_equivalence_representable_single_linkage} implies that the computation of ultrametrics arising from representable methods requires a symmetrizing operation that depends on $\Omega$ followed by application of a single linkage algorithm; see, e.g., \cite{Gabowetal86}. 
A related decomposition result is derived in \cite[Theorem 6.3]{CarlssonMemoli10} for clustering in metric spaces. 
Proposition \ref{prop_equivalence_representable_single_linkage} is a significant extension of this result which applies not only to finite metric spaces but also to asymmetric networks in general. 
For the case of metric spaces, when sensitivity to density might be a desirable property, suitable choices of representers $\Omega$ are known to induce this behavior; see~\cite[Section 6.7 and Definition 7.2]{CarlssonMemoli10}.

The converse of Proposition~\ref{prop_equivalence_representable_single_linkage} is not true, i.e., the composition of \emph{any} symmetrizing map followed by single linkage need not correspond to a representable method for some family of representers $\Omega$. 
To see this, consider the grafting method $\ccalH^{\R/\NR}(\beta)$ introduced in Section~\ref{Ss:linear_scale_preservation}. 
We can think of the application of $\ccalH^{\R/\NR}(\beta)$ as a symmetrizing map ($\ccalH^{\R/\NR}(\beta)$ itself) followed by single linkage.
In this case, the application of single linkage would be moot, since the image of $\ccalH^{\R/\NR}(\beta)$ is already an ultrametric (this follows from the fact that single linkage attains the maximal sub-dominant ultrametric~\cite[Corollary 14]{clust-um}).
Thus, we have argued that $\ccalH^{\R/\NR}(\beta)$ can be decomposed as in Fig.~\ref{fig_decomposition_of_representable_methods}, however $\ccalH^{\R/\NR}(\beta)$ is not representable since we have shown that it is not linear scale preserving and that would violate our main result in Theorem~\ref{theo_representability_excisiveness}.

\section{A generative model for robust hierarchical clustering methods}\label{subsec_generative_model}

Our main theorem establishes the equivalence between the classes of representable and robust hierarchical clustering methods.

%
\begin{theorem}\label{theo_representability_excisiveness}
	Given an admissible hierarchical clustering method $\ccalH$, it is robust (P1), (P2), (P3) if and only if it is representable (P4).
\end{theorem}

Intuitively, the relationship between representability and robustness stated in Theorem \ref{theo_representability_excisiveness} originates from the fact that both concepts address the locality of clustering methods. 
Representability (P4) implies that the method can be interpreted as an extension of particular cases or representers. 
In a related fashion, excisiveness (P1) requires the clustering of local subnetworks to be consistent with the clustering of the entire network. 

%
\begin{figure}
	\centering
	\includegraphics[width=0.75\linewidth]{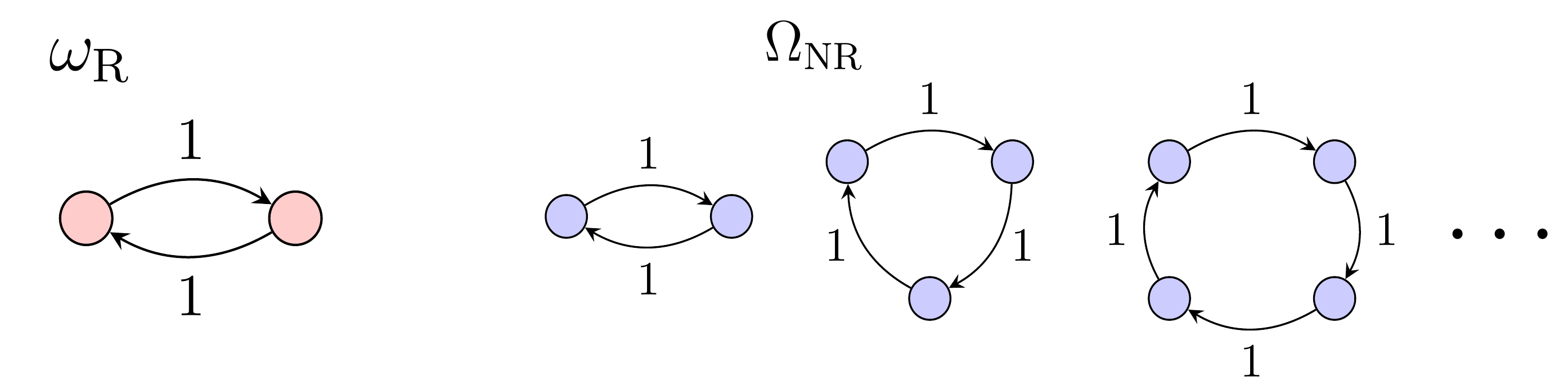}
	\vspace{-4mm}
	\caption{{$\ccalH^{\R}$ can be represented by one representer $\omega_\R$ while $\ccalH^\NR$ requires a countably infinite collection $\Omega_{\NR}$ of representers.}}
	\label{fig_reciprocal_nonreciprocal_representable}
\end{figure}

The importance of Theorem \ref{theo_representability_excisiveness} resides in relating implicit properties of a clustering method with practical relevance such as linear scale preservation with a generative model of clustering methods such as representability. 
Thus, when designing a clustering method for a particular application, if robustness is a desirable property, then Theorem \ref{theo_representability_excisiveness} asserts that representability must be considered as a generative model. Conversely, it is unclear how to establish directly whether a given clustering method is representable. However, Theorem \ref{theo_representability_excisiveness} provides an indirect way to prove representability via the analysis of excisiveness, linear scale preservation, and stability.

In Section~\ref{S:robust}, we described an admissible method (grafting) which is not linear scale preserving and another one (semi-reciprocal clustering) which is not excisive. Hence, Theorem \ref{theo_representability_excisiveness} states that neither of these methods is representable. Conversely, by combining Theorem \ref{theo_representability_excisiveness} with Propositions \ref{prop_reciprocal_nonreciprocal_excisive}, \ref{prop_reciprocal_nonreciprocal_linear_scale_preserving}, and~\ref{prop_reciprocal_nonreciprocal_stable}, we can ensure that the reciprocal $\ccalH^\R$ and nonreciprocal $\ccalH^{\NR}$ methods are both representable. 
Indeed, in Fig. \ref{fig_reciprocal_nonreciprocal_representable} we exhibit the collections of representers associated with each of the two methods, i.e. $\ccalH^\R \equiv \ccalH^{\omega_\R}$ and $\ccalH^{\NR} \equiv \ccalH^{\Omega_\NR}$. 

To see why the equivalence stated in Theorem \ref{theo_representability_excisiveness} is true for the case of reciprocal clustering, pick an arbitrary network $N=(X, A_X)$ and notice that the expansion constant [cf. \eqref{eqn:def_lipschitz_constant}] of any map $\phi$ from $\omega_\R$ to $N$ is equal to 
\begin{equation}\label{eqn_lipschitz_constant_reciprocal}
L(\phi; \omega_\R, N) = \max \big( A_X(\phi(z), \phi(z')), A_X(\phi(z'), \phi(z))\big),
\end{equation}
where $z$ and $z'$ denote the two nodes of the representer $\omega_\R$. Moreover, from the definition of optimal multiple between nodes $x, x' \in X$, we know that nodes $x$ and $x'$ must be the images of $z$ and $z'$ under $\phi$ which implies that
\begin{equation}\label{eqn_optimal_multiple_reciprocal}
\lambda^{\omega_\R}_X(x, x') = \max( A_X(x, x'), A_X(x', x)).
\end{equation}
By combining \eqref{eqn_optimal_multiple_reciprocal} and \eqref{eqn_def_representability_single_network} and comparing this with the definition of reciprocal clustering \eqref{eqn_reciprocal_clustering}, it follows that $\ccalH^\R \equiv \ccalH^{\omega_\R}$.

Similarly, to see why the equivalence $\ccalH^\NR \equiv \ccalH^{\Omega_\NR}$ is true, notice that for a pair of arbitrary nodes $x, x' \in X$, we may concatenate two minimizing chains $C(x, x')$ and $C(x', x)$ that achieve the minimum directed costs $\tdu^*_X(x, x')$ and $\tdu^*_X(x', x)$ respectively [cf. \eqref{eqn_nonreciprocal_chains}] to obtain a loop. The maximum dissimilarity in this loop is equal to $\max(\tdu^*_X(x, x'), \tdu^*_X(x', x))$ which is exactly $u^\NR_X(x, x')$ [cf. \eqref{eqn_nonreciprocal_clustering}].  Furthermore, if this loop is composed of $k$ nodes, then we may pick the representer in $\Omega_\NR$ with exactly $k$ nodes and map it injectively to the loop. Since by construction $x$ and $x'$ belong to the image of the map and its expansion constant is equal to the maximum dissimilarity in the loop $u^\NR_X(x, x')$, we obtain that $\lambda^{\Omega_\NR}_X(x, x')=u^\NR_X(x, x')$ from which the result follows.

In general, one can design representers $\Omega$ different from those in Fig.~\ref{fig_reciprocal_nonreciprocal_representable} to capture diverse structures in the directed network under study, thus leading to representable methods that go beyond reciprocal and nonreciprocal clustering. An example of this is given in Section~\ref{sec_numerical_experiments}.
Moreover, regardless of the particular choice of $\Omega$, Theorem~\ref{theo_representability_excisiveness} guarantees that the resulting hierarchical clustering method will be robust. 
This is also illustrated through a numerical experiment in the next section.

%
\begin{figure}
	\centering
	\includegraphics[width = 0.75\textwidth]{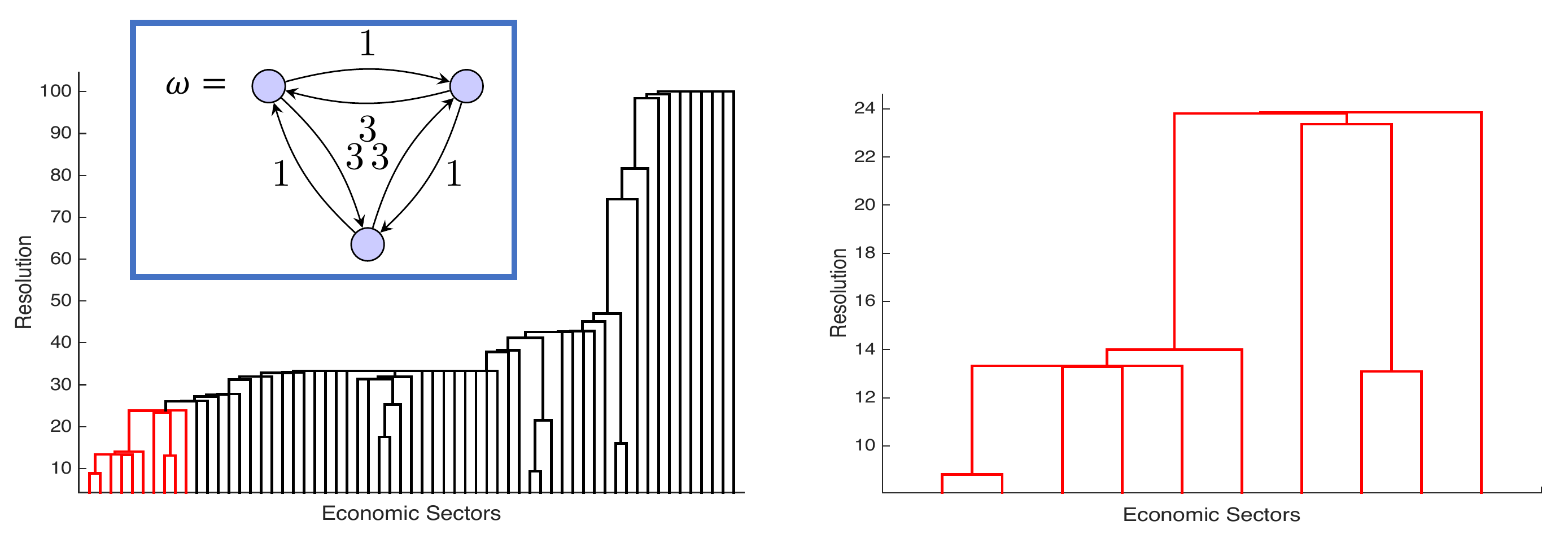}
	\vspace{-0.1in}
	\caption{(Left) Dendrogram obtained when clustering the economic network $N_I$ using the representable clustering method $\ccalH^\omega$, for the representer $\omega$ shown. (Right) Illustration of excisiveness. When clustering the subnetwork spanned by the economic sectors corresponding to the red branch in the left, the output dendrogram matches the branch.}
	\label{fig_numerical_experiments_together}
\end{figure}

\section{Experimental illustration}\label{sec_numerical_experiments}

The U.S. Department of Commerce publishes a yearly table of inputs and outputs organized by economic sectors\footnote{Available at \url{http://www.bea.gov/industry/io_annual.htm}}.
We focus on a specific section of this table, called \emph{uses}, that corresponds to the inputs to production for different industrial sectors. 
More precisely, we are given a set $I$ of 61 industrial sectors as defined by the North American Industry Classification System and a similarity function $U\!:\! I \! \times \! I \to \reals_+$ where $U(i, i')$ represents how much of the production of sector $i$ (in dollars) is used as an input of sector $i'$. Based on this, we define the network $N_I=(I, A_I)$ where the dissimilarity function $A_I$ satisfies $A_I(i,i)=0$ for all $i\in I$ and, for $i \neq i' \in I$, is given by
\begin{equation}\label{eqn_def_io_dissimilarity}
A_I(i, i') := \left(\frac{U(i, i')}{\sum_k U(k, i')}\right)^{-1}.
\end{equation}

The normalization in \eqref{eqn_def_io_dissimilarity} can be interpreted as the proportion of the input to productive sector $i'$ that comes from sector $i$. Consequently, we focus on the relative combination of inputs of a sector rather than the size of the economic sector itself. Moreover, we compute the inverse of this normalized quantity to obtain a measure $A_I$ that represents dissimilarities. I.e., if most of the productive input of $i'$ comes from $i$, then the normalization would output a number close to 1 and the dissimilarity measure $A_I(i, i')$ would be small.

We hierarchically cluster the network $N_I$ of economic sectors using the representable method $\ccalH^\omega$ associated with the representer $\omega$ in Fig.~\ref{fig_numerical_experiments_together}-(left); see Section~\ref{subsec_implementation_details} for details. 
From the structure of $\omega$, the method $\ccalH^\omega$ clusters two nodes if they can be joined via cycles of at most three nodes with strong connection in one direction -- represented by the dissimilarities equal to 1 --  while simultaneously having not too weak connections in the opposite direction -- represented by the dissimilarities equal to~3.

In Fig.~\ref{fig_numerical_experiments_together}-(left), we present the output dendrogram when the method $\ccalH^\omega$ is applied to $N_I$. Implementation details of this particular clustering method can be found in Section~\ref{subsec_implementation_details}. Theorem \ref{theo_representability_excisiveness} guarantees that if we take a branch of the dendrogram in Fig.~\ref{fig_numerical_experiments_together}-(left), e.g. the one highlighted in red, and focus on a subnetwork of the economic network spanned by the corresponding industrial sectors and cluster this subnetwork, we obtain a dendrogram equivalent to the red branch. Indeed, this is the case as can be seen in Fig.~\ref{fig_numerical_experiments_together}-(right). Similarly, we can multiply the economic network by a scalar and cluster the resulting multiple network and we are guaranteed to obtain a multiple of the original dendrogram [cf.~(P2)], and small perturbations to the network result in small perturbations to the output dendrogram [cf.~(P3)].

\subsection{Implementation of the representable method $\ccalH^\omega$}\label{subsec_implementation_details}

First notice that for an arbitrary network $N_X=(X, A_X)$, the disimilarity function $A_X$ can be represented as a matrix which, as it does not lead to confusion, we also denote as $A_X\in\reals^{n\times n}$. Define the matrix $B_X$ where each element is given by
\begin{align}\label{eqn_app_applying_clustering_algorithm_010}
[B_X]_{ij} =  \min_k  \max \Big(  [A_X]_{ij},[A_X]_{jk},[A_X]_{ki}, [A_X]_{ji}/3,[A_X]_{kj}/3,[A_X]_{ik}/3 \Big).
\end{align}
By comparing \eqref{eqn_app_applying_clustering_algorithm_010} with \eqref{eqn:def_lipschitz_constant}, it follows that the element $i,j$ of matrix $B_X$ stores the minimum of the expansion constant of a map $\phi$ from the representer $\omega$ to the network $N_X$ with nodes $i$ and $j$ in its image \emph{and} mapping a unit dissimilarity in $\omega$ to the directed dissimilarity from $i$ to $j$. 
From this interpretation of $B_X$ it follows immediately that the symmetric matrix
$\Lambda_X := \min (B_X, B_X^T)$ contains as elements the optimal multiples, i.e. $[\Lambda_X]_{i,j}=\lambda^{\omega}_X(i,j)$. To see this, notice that the optimal map from $\omega$ to $N_X$ attaining the minimum expansion constant in \eqref{eqn:multiple_as_lipschitz_constant} must contain nodes $i$ and $j$ in its image and must map a unit dissimilarity in $\omega$ either to the directed dissimilarity from $i$ to $j$ or from $j$ to $i$, thus $\lambda^{\omega}_X(i,j) = \min( [B_X]_{ij}, [B_X]_{ji})$.

Finally, we compute the output ultrametric as in \eqref{eqn_def_representability_single_network}, which is equivalent to applying single linkage clustering to the symmetric network $(X, \Lambda_X)$, thus any known single linkage algorithm \cite{Gabowetal86} can be used for this last step.

\section{Conclusion}\label{sec_conclusion}

We defined \emph{robustness} of hierarchical clustering methods via the fulfillment of three properties: excisiveness (the clustering output of a subnetwork does not depend on the information beyond the subnetwork), linear scale preservation (the clustering output is not modified by a change of units), and stability (a small perturbation in the network entails a small perturbation in the clustering output).  
As a generative model for hierarchical clustering methods we introduced the concept of representability. The behavior of representable methods is determined by specifying their output on a collection of representers. 
Moreover, we showed that every representable method can be decomposed into two phases: a symmetrizing map $\Lambda^\Omega$ followed by single linkage clustering.
This decomposition result enables the decoupled implementation of hierarchical clustering methods of practical relevance.
Our main result was the proof that, within the set of admissible hierarchical clustering methods, the subset of representable methods coincides with the class of robust methods as determined by the three aforementioned properties.

For future work, it seems interesting to understand how the complexity of computing $\Lambda^\Omega$ depends on the structure of the collection of associated representers $\Omega$. It also seems of particular interest to expand the list of desirable practical properties in order to get a more stringent characterization of the methods that are relevant in practice. We can consider, e.g., the notion of scale preservation for general dissimilarity transformations \emph{not} restricted to linear transformations as done in this paper. One overreaching aim is to identify conditions that need to be imposed on the representers so that the associated representable method complies with the stricter notion of practicality.

\appendix

\section{Relegated proofs}

\subsection{Proof of Proposition~\ref{prop_validity_repre}}

Given a collection $\Omega$ of representers $\omega=(X_{\omega}, A_{\omega})$, we want to see that for an arbitrary network $N=(X, A_X)$ the output $(X, u^\Omega_X)=\ccalH^{\Omega}(X,A_X)$ satisfies the identity, symmetry, and strong triangle inequality properties of an ultrametric. To show that the strong triangle inequality in \eqref{eqn_strong_triangle_inequality} is satisfied let $C^*(x,x')$ and $C^*(x',x'')$ be minimizing chains for $u^\Omega_X(x, x')$ and $u^\Omega_X(x', x'')$, respectively. Consider then the chain $C(x,x'')$ obtained by concatenating $C^*(x,x')$ and $C^*(x',x'')$, in that order. Notice that the maximum over $i$ of the optimal multiples $\lambda^{\Omega}_X(x_i, x_{i+1})$ in $C(x,x'')$ does not exceed the maximum multiples in each individual chain. Thus, the maximum multiple in the concatenated chain $C(x,x'')$ suffices to bound $u^{\Omega}_X(x,x'') \leq \max \big( u^{\Omega}_X(x,x'), u^{\Omega}_X(x',x'')\big)$ by \eqref{eqn_def_represent_1} as in \eqref{eqn_strong_triangle_inequality}. 
	
	To show the symmetry property, $u^\Omega_X(x, x') = u^\Omega_X(x', x)$ for all $x, x' \in X$, first notice that a direct implication of the definition of optimal multiples in \eqref{eqn:multiple_as_lipschitz_constant} is that $\lambda^\omega_X(x, x')=\lambda^\omega_X(x', x)$ for every representer $\omega$. From \eqref{eqn_Omega_multiple_from_omega} we then obtain that $\lambda^\Omega_X$ is symmetric, from where symmetry of $u^\Omega_X$ immediately follows.
	
	For the identity property, i.e. $u^\Omega_X(x, x') = 0$ if and only if $x=x'$, we first show that if $x=x'$ we must have $u_X^\Omega(x,x')=0$. Pick any $x \in X$, let $x'=x$ and pick the chain $C(x, x)=[x, x]$ starting and ending at $x$ with no intermediate nodes as a candidate minimizing chain in \eqref{eqn_def_represent_1}. While this particular chain need not be optimal in \eqref{eqn_def_represent_1} it nonetheless holds that
	\begin{equation}\label{eqn_ultram_identity_x_x_rep}
	0 \leq u^{\Omega}_X(x, x) \leq \lambda^{\Omega}_X(x, x),
	\end{equation}
	where the first inequality holds because all costs $\lambda^{\Omega}_X(x_i, x_{i+1})$ in \eqref{eqn_def_represent_1} are non-negative since they correspond to the expansion constant of some map, which is non-negative by definition \eqref{eqn:def_lipschitz_constant}. Notice that for the cost $\lambda^{\omega}_X(x, x)$ in \eqref{eqn:multiple_as_lipschitz_constant}, we minimize the expansion constant among maps $\phi_{x, x}$ that are only required to have node $x$ in its image. Thus, consider the map that takes all the nodes in any representer $\omega \in \Omega$ into node $x \in X$. From \eqref{eqn:def_lipschitz_constant}, the expansion constant of this map is zero which implies by \eqref{eqn:multiple_as_lipschitz_constant} that $\lambda^{\omega}_X(x, x)=0$ for all $\omega \in \Omega$. Combining this result with \eqref{eqn_Omega_multiple_from_omega} we then get that $\lambda^{\Omega}_X(x, x) = 0$ and from \eqref{eqn_ultram_identity_x_x_rep} we conclude that $u^{\Omega}_X(x, x) = 0$.
	
	In order to show that the condition $u_X^\Omega(x,x')=0$ implies that $x=x'$ we prove that if $x\neq x'$
	we must have $u_X^\Omega(x,x') > \alpha>0$ for some strictly positive constant $\alpha$. In proving this, we make use of the following claim.
	
	%
	\begin{claim}\label{lemma_weak_connect}
		Given a network $N=(X, A_X)$, a weakly connected representer $\omega=(X_{\omega}, A_{\omega})$, and a dissimilarity reducing map $\phi: X_{\omega} \to X$ whose image satisfies $|\text{Im}(\phi)| \geq 2$, there exists a pair of points $(z, z') \in \text{dom}(A_{\omega})$ for which $\phi(z) \neq \phi(z')$.
	\end{claim}
	\begin{proof}
		Suppose that $\phi(z^1)=x^1$ and $\phi(z^2)=x^2$, with $x^1 \neq x^2 \in X$. These nodes can always be found since $|\text{Im}(\phi)| \geq 2$. By our hypothesis, the network is weakly connected. Hence, there must exist a chain $C(z^1, z^2) = [ z^1=z_0, z_1,\ldots, z_l=z^2]$ linking $z^1$ and $z^2$ for which either $(z_i, z_{i+1}) \in \text{dom}(A_\omega)$ or $(z_{i+1}, z_{i}) \in \text{dom}(A_\omega)$ for all $i=0,\ldots,l-1$. Focus on the image of this chain under the map $\phi$, $C(x^1, x^2) = [ x^1= \phi(z_0), \phi(z_1),\ldots, \phi(z_l)=x^2]$. Notice that not all the nodes are necessarily distinct, however, since the extreme nodes are different by construction, at least one pair of consecutive nodes must differ, say $\phi(z_p) \neq \phi(z_{p+1})$. Due to $\omega$ being weakly connected, in the original chain we must have either $(z_p, z_{p+1})$ or $(z_{p+1}, z_{p}) \in \text{dom}(A_{\omega})$. Hence, either $z=z_p$ and $z'=z_{p+1}$ or vice versa must fulfill the statement of the claim.
	\end{proof}
	
	Returning to the main argument, observe that since pairwise dissimilarities in all networks $\omega \in \Omega$ are uniformly bounded, the maximum dissimilarity across all links of all representers 
	\begin{equation}\label{eqn_max_dissim_omega}
	d_{\max} = \sup_{\omega \in \Omega} \,\, \max_{(z, z') \in \text{dom}(A_{\omega})} A_{\omega}(z, z'),
	\end{equation} 
	is guaranteed to be finite. Define the separation of the network as its minimum positive dissimilarity, i.e., $\sep(X,A_X) := \min_{x \neq x'} A_X(x, x')$ and pick any real $\alpha$ such that $0 < \alpha < \sep(X, A_X)/d_{\max}$. Then for all $(z, z') \in \text{dom}(A_{\omega})$ and all $\omega \in \Omega$ we have
	\begin{equation}\label{eqn_dissim_separ_2}
	\alpha \ A_{\omega}(z, z') < \sep(X, A_X).
	\end{equation}
	Claim \ref{lemma_weak_connect} implies that regardless of the map $\phi$ chosen, this map transforms some defined dissimilarity in $\omega$, i.e. $A_{\omega}(z, z')$ for some $(z, z') \in \text{dom}(A_{\omega})$, into a dissimilarity in $N$. Moreover, every positive dissimilarity in $N$ is greater than or equal to the network separation $\sep(X, A_X)$. Hence, \eqref{eqn_dissim_separ_2} implies that there cannot be any dissimilarity reducing map $\phi$ with $|\text{Im}(\phi)| \geq 2$ from $\alpha * \omega$ to $N$ for any $\omega \in \Omega$. From \eqref{eqn:multiple_as_lipschitz_constant}, this implies that for all $x \neq x' \in X$ and for all $\omega$ we have that $\lambda^{\omega}_X(x, x') > \alpha > 0$. Hence, from \eqref{eqn_Omega_multiple_from_omega} we conclude that $\lambda^{\Omega}_X(x, x') >\alpha>0$, which in turn implies that the ultrametric value between two different nodes $u^{\Omega}_X(x, x')$ must be strictly positive.

\subsection{Proof of Theorem~\ref{theo_representability_excisiveness}}

We first prove that (P4) implies (P1)-(P3). 
Notice that the expansion constants of arbitrary maps \eqref{eqn:def_lipschitz_constant} satisfy
\begin{equation}\label{eqn:lipschitz_constant_linear}
L(\phi; \omega, \alpha*N) = \alpha \, L(\phi; \omega, N),
\end{equation}
for any positive constant $\alpha > 0$. 
That (P4) implies (P2) follows by combining the linear relation in \eqref{eqn:lipschitz_constant_linear} with the definition of a representable method in \eqref{eqn_def_represent_1}.

To show that representability implies excisiveness (P1), we must prove that \eqref{eqn_def_excisiveness} is true for a general representable clustering method $\ccalH^\Omega$. Hence, consider a network $N=(X, A_X)$, a resolution $\delta > 0$ and a subnetwork $N^\delta_i = \big(B_i(\delta),\ A_X \big|_{B_i(\delta) \times B_i(\delta)}\big)$ as defined in \eqref{eqn_excisiveness_subnetworks}, and define the output ultrametrics $(X, u^{\Omega}_X)=\ccalH^\Omega(N)$ and $(X, u^{\Omega}_{N^\delta_i})=\ccalH^\Omega(N^\delta_i)$. Since the identity map from $N^\delta_i$ to $N$ is dissimilarity reducing, admissibility of $\ccalH^\Omega$ implies [cf. Axiom of Transformation (A2)] 
\begin{equation}\label{eqn_representability_implies_excisiveness_00}
u^{\Omega}_{N^\delta_i}(x, x') \geq u_X^{\Omega}(x, x'),
\end{equation}
for all $x, x' \in B_i(\delta)$. 
In order to show the reverse inequality, pick arbitrary nodes $x, x' \in B_i(\delta)$. From the definition of subnetwork \eqref{eqn_excisiveness_subnetworks_ultrametrics}, it must be that 
\begin{align}\label{eqn_representability_implies_excisiveness}
	u^{\Omega}_X(x, x') \leq \delta, \qquad
	u^{\Omega}_X(x, x'') > \delta,
\end{align}
for all $x'' \not\in B_i(\delta)$.
The leftmost inequality in \eqref{eqn_representability_implies_excisiveness} implies that there exists a minimizing chain $C(x, x')=[x=x_0, x_1, ... , x_l=x']$ in definition \eqref{eqn_def_represent_1} and a series of maps $\phi_{x_j, x_{j+1}}$ for all $j$ determining the optimal multiples $\lambda^\Omega_X(x_j, x_{j+1}) \leq \delta$. Notice that the ultrametric value between any two nodes in the images of the maps $\phi_{x_j, x_{j+1}}$ is smaller than or equal to $\delta$. Hence, from \eqref{eqn_representability_implies_excisiveness} we have that the minimizing chain $C(x, x')$ and the image of every optimal dissimilarity reducing map is contained in $B_i(\delta)$ so that the same chain can be used to compute $u^{\Omega}_{N^\delta_i}(x, x')$. This implies that
\begin{equation}\label{eqn_representability_implies_excisiveness_000}
u^{\Omega}_{N^\delta_i}(x, x') \leq u_X^{\Omega}(x, x'),
\end{equation}
for all $x, x' \in B_i(\delta)$.
Combining \eqref{eqn_representability_implies_excisiveness_00} with \eqref{eqn_representability_implies_excisiveness_000} we obtain \eqref{eqn_def_excisiveness}, showing that (P4) implies (P1).

{To} prove that (P4) implies (P3), we resort to Proposition~\ref{prop_equivalence_representable_single_linkage} where we have that $\ccalH^\Omega \equiv \ccalH^\SL\circ \Lambda^\Omega.$ In \cite{clust-um} it was shown that 
$d_{\ccalN}(\ccalH^\SL(N_X),\ccalH^\SL(N_Y))\leq d_\ccalN(N_X,N_Y)$, for any $N_X$ and $N_Y$ in $\ccalN$. Thus, in order to establish our claim it is enough to prove that there exists a finite constant $L=L(\Omega)\geq 0$ such that
\begin{equation} 
d_{\ccalN}(\Lambda^\Omega(N_X),\Lambda^\Omega(N_Y))\leq L \, d_\ccalN(N_X,N_Y).
\end{equation}
We claim this to be true for $L(\Omega):=\big(\mathrm{sep}(\Omega)\big)^{-1}$.

In order to verify this, assume that $\eta=d_{\mathcal{N}}(N_X,N_Y)$ and pick any correspondence $R$ between $X$ and $Y$ such that $|A_X(x,x')-A_Y(y,y')|\leq 2\eta$ for all  $(x,y)$ and $(x',y')$ in $R$ [cf. \eqref{eqn_gh_distance}]. Fix any two pairs $(x,y)$ and $(x',y')$ in $R$.
For any representer $\omega\in \Omega$, let $\phi:\omega\rightarrow X$ be any map such that $x,x'\in\mathrm{Im}(\phi)$. Moreover, consider any function $\varphi:X\rightarrow Y$ such that $\varphi(x)=y$ and $\varphi(x')=y'$ and $(x'',\varphi(x''))\in R$ for all $x''\in X$. Notice that the definition of correspondence ensures that at least one such function $\varphi$ exists. Then, we have
\begin{align}\label{eqn_proof_stability_50}
	L(\varphi\circ\phi;\omega,N_Y) \! \leq \! \max_{\substack{(z, z') \in \text{dom}(A_\omega) \\ z \neq z'}} \!\!\!\!\!\frac{A_X(\phi(z), \phi(z'))}{A_\omega(z, z')} + 2\eta \,\, \mathrm{sep}(\omega)^{-1}
	 = L(\phi;\omega,N_X) + 2\eta \,\, \mathrm{sep}(\omega)^{-1}.
\end{align}
By construction, $y,y'\in \mathrm{Im}(\varphi\circ\phi)$. Thus, $L(\varphi\circ\phi;\omega,N_Y)$ is an upper bound for the optimal multiple $\lambda_Y^\omega(y,y')$ so from~\eqref{eqn_proof_stability_50} it follows that
\begin{equation}\label{eqn_proof_stability_60}
\lambda_Y^\omega(y,y')\leq L(\phi;\omega,N_X)+2\eta \,\, \mathrm{sep}(\omega)^{-1}.
\end{equation}
This inequality is valid for all functions $\phi:\omega\rightarrow X$ s.t. $x,x'\in\mathrm{Im}(\phi)$. Thus, for the particular map $\phi$ minimizing $L(\phi;\omega,N_X)$, \eqref{eqn_proof_stability_60} becomes $\lambda_Y^\omega(y,y')\leq \lambda_X^\omega(x,x')+2\eta \,\,\, \mathrm{sep}(\omega)^{-1}$. By symmetry, we obtain $|\lambda_X^\omega(x,x')- \lambda_Y^\omega(y,y')|\leq 2\eta \,\,\, \mathrm{sep}(\omega)^{-1}$, for all $(x,y),(x',y')\in R$. It then follows that 
\begin{equation}\label{eqn_proof_stability_100}
|\lambda_X^\Omega(x,x')- \lambda_Y^\Omega(y,y')|\leq 2\eta \,\,\, \mathrm{sep}(\Omega)^{-1},
\end{equation}
as claimed, where the fact that we require $\mathrm{sep}(\Omega) > 0$ guarantees that \eqref{eqn_proof_stability_100} is well-defined.
This completes the proof that (P4) implies (P1)-(P3).

To prove the converse statement, consider an arbitrary admissible clustering method $\ccalH$ which is excisive, linear scale preserving, and stable. 
We will construct a representable method $\ccalH^\Omega$ such that $\ccalH \equiv \ccalH^\Omega$.

Denote by $(X, u_X)=\ccalH(X, A_X)$ an arbitrary output ultrametric and define the collection of representers $\Omega$ as follows:

\begin{align}\label{eqn_definition_omega_excisiveness}
	\Omega = \left\{\, \omega \,\, \Big| \,\, \omega = \frac{1}{\!\!\!\underset{x, x' \in B_i(\delta)}{\max} u_X(x, x')} * N^\delta_i ,
	|B_i(\delta)| > 1, \delta > 0 \right\},
\end{align}
for all resolutions $\delta > 0$ and $N^\delta_i := (B_i(\delta), A_X|_{B_i(\delta) \times B_i(\delta)})$ being a subnetwork of all possible networks $N=(X, A_X)$ given the method $\ccalH$. In other words, we pick as representers the collection of all possible subnetworks generated by the method $\ccalH$, each of them scaled by the inverse of the maximum ultrametric obtained in such subnetwork. Notice that from the definition of subnetwork \eqref{eqn_excisiveness_subnetworks_ultrametrics} we have that
\begin{equation}\label{eqn_excisiveness_implies_representability}
\max_{x, x' \in B_i(\delta)} u_X(x, x') \leq \delta,
\end{equation}
which appears in the denominator of the definition \eqref{eqn_definition_omega_excisiveness} for every representer $\omega \in \Omega$.

We show equivalence of methods $\ccalH$ and $\ccalH^\Omega$ by showing that the ultrametric outputs coincide for every network. Pick an arbitrary network $N=(X, A_X)$ and two different nodes $x, x' \in X$ and define $\alpha := u_X(x,x')$. Since $\Omega$ was built considering all possible networks, including $N$, there is a representer $\omega \in \Omega$ that corresponds to the subnetwork $N^\alpha_i$ at resolution $\alpha$ that contains $x$ and $x'$. From \eqref{eqn_excisiveness_implies_representability}, the inclusion map $\phi$ from $\alpha * \omega$ to $N$ such that $\phi(x)=x$ is dissimilarity reducing and $x, x' \in \text{Im}(\phi)$. From definition \eqref{eqn:multiple_as_lipschitz_constant} this implies that $\lambda^{\omega}_X(x, x') \leq \alpha$. By substituting in \eqref{eqn_Omega_multiple_from_omega} and further substitution in \eqref{eqn_def_represent_1} we obtain that $u^\Omega_X(x, x') \leq \alpha$. Recalling that $\alpha = u_X(x,x')$ and that we chose the network $N$ and the pair of nodes $x, x'$ arbitrarily, we may conclude that $u^{\Omega}_X \leq u_X$, for every network $N$. 

In order to show the other direction of the inequality, we must first observe that for every representer, the ultrametric value given by $\ccalH$ between any pair of nodes in the representer is upper bounded by $1$. To see this, given a representer $\omega=(X_\omega, A_{X_\omega})$ associated with the subnetwork $N^\delta_i$ in \eqref{eqn_definition_omega_excisiveness} we have that 
\begin{align}\label{eqn_excisiveness_implies_representability_2}
	u_{X_\omega}(\tdx, \tdx') \! = \! \frac{1}{\underset{x, x' \in B_i(\delta)}{\max} u_X(x, x')} \,\, u_{B_i(\delta)}(\tdx, \tdx')
	\!
	= \! \frac{1}{\underset{x, x' \in B_i(\delta)}{\max} u_X(x, x')} \,\, u_{X}|_{B_i(\delta) \times B_i(\delta)}(\tdx, \tdx') \leq 1, \nonumber
\end{align}
for all $\tdx, \tdx' \in X_\omega$.
The first equality in \eqref{eqn_excisiveness_implies_representability_2} is implied by the definition of $\omega$ in \eqref{eqn_definition_omega_excisiveness} and linear scale preservation of $\ccalH$. The second equality is derived from excisiveness of $\ccalH$.

Pick an arbitrary network $N=(X, A_X)$ and a pair of nodes $x, x' \in X$ and define $\beta := u^{\Omega}_X(x, x')$. This means that there exists a minimizing chain $C(x, x')=[x'=x_0, x_1, ... , x_l=x']$ such that for every consecutive pair of nodes we can find a dissimilarity reducing map $\phi_{x_j, x_{j+1}}$ from $\beta * \omega_j$ to $N$ for some representer $\omega_j \in \Omega$ such that $x_j, x_{j+1} \in \text{Im}(\phi_{x_j, x_{j+1}})$. Focus on a particular pair of consecutive nodes $x_j, x_{j+1}$ and denote by $p_j, p_{j+1}$ two respective pre-images on $\omega_j=(X_{\omega_j}, A_{X_{\omega_j}})$ under the map $\phi_{x_j, x_{j+1}}$. Without loss of generality, we can assume that $x_j \neq x_{j+1}$ for all $j$. The pre-images need not be unique. Denote by $\beta * \omega_j=(X^\beta_{\omega_j}, \beta \, A_{X_{\omega_j}})$ the $\beta$ multiple of the representer $\omega_j$. Since $\phi_{x_j, x_{j+1}}$ is a dissimilarity reducing map from $\beta* \omega_j$ to $N$, the Axiom of Transformation (A2) implies that
\begin{equation}\label{eqn_excisiveness_implies_representability_3}
u_{X^\beta_{\omega_j}}(p_j, p_{j+1}) \geq u_X(x_j, x_{j+1}).
\end{equation}
Moreover, we can assert that
\begin{equation}\label{eqn_excisiveness_implies_representability_4}
u_{X^\beta_{\omega_j}}(p_j, p_{j+1}) = \beta \, u_{X_{\omega_j}}(p_j, p_{j+1}) \leq \beta,
\end{equation}
where the equality is due to linear scale preservation and the inequality is justified by \eqref{eqn_excisiveness_implies_representability_2}. From the combination of \eqref{eqn_excisiveness_implies_representability_3} and \eqref{eqn_excisiveness_implies_representability_4} we obtain that $u_X(x_j, x_{j+1}) \leq \beta$. 
Since this is true for an arbitrary pair of consecutive nodes in $C(x, x')$, from the strong triangle inequality we have that $u_X(x, x') \leq \max_j u_X(x_j, x_{j+1}) \leq \beta$.
Recalling that $\beta = u^{\Omega}_X(x, x')$ and that the network $N$ was arbitrary, we can conclude that $u^{\Omega}_X \geq u_X$, for every network $N=(X, A_X)$. Combining this with $u^{\Omega}_X \leq u_X$, we conclude that $u^{\Omega}_X = u_X$, completing the proof.

\vspace{-1mm}

\section*{Acknowledgment}
The authors would like to thank Alejandro Ribeiro for his valuable technical input and his thoughtful comments on an earlier version of this manuscript.

\vspace{-2mm}

\bibliographystyle{siamplain}
\bibliography{clustering_biblio}

\end{document}